\documentclass{article}

\usepackage{microtype}
\usepackage{graphicx}
\usepackage{subfigure}
\usepackage{booktabs} 
\usepackage[utf8]{inputenc} 
\usepackage[T1]{fontenc}    

\usepackage{url}            
\usepackage{booktabs}       
\usepackage{amsfonts}       
\usepackage{nicefrac}       

\usepackage[english]{babel}

\usepackage{color}

\usepackage{amsthm}
\usepackage{amsmath}
\usepackage{mathtools}
\usepackage{setspace}
\usepackage{algpseudocode}
\usepackage{caption} 
\makeatletter
\newcommand{\leqnomode}{\tagsleft@true}
\newcommand{\reqnomode}{\tagsleft@false}
\makeatother

\usepackage{amssymb}
\usepackage{tcolorbox}

\usepackage{hyperref}

\newtheorem{lemma}{Lemma}

\newtheorem{defi}{Definition}

\newtheorem{theorem}{Theorem}

\newtheorem{assumption}{Assumption}
\newcommand{\red}[1]{{\color{red}{#1}}}
\newcommand{\blue}[1]{{\color{blue}{#1}}}

\newcommand\numberthis{\addtocounter{equation}{1}\tag{\theequation}}

\newcommand\myeq{\stackrel{\mathclap{\tiny\mbox{i.i.d.}}}{\sim}}

\newcommand{\FullTitle}{Sharp-MAML: Sharpness-Aware Model-Agnostic Meta Learning}

\usepackage[accepted]{icml2022}

\icmltitlerunning{Sharp-MAML: Sharpness-Aware Model-Agnostic Meta Learning}

\begin{document}

\twocolumn[
\icmltitle{\FullTitle}



\icmlsetsymbol{equal}{*}

\begin{icmlauthorlist}
\icmlauthor{Momin Abbas}{equal,rpi}
\icmlauthor{Quan Xiao}{equal,rpi}
\icmlauthor{Lisha Chen}{equal,rpi}
\icmlauthor{Pin-Yu Chen}{ibm}
\icmlauthor{Tianyi Chen}{rpi}
\end{icmlauthorlist}

\icmlaffiliation{rpi}{Rensselaer Polytechnic Institute, Troy, NY}
\icmlaffiliation{ibm}{IBM Thomas J. Watson Research Center, NY, USA}

\icmlcorrespondingauthor{Tianyi Chen}{chentianyi19@gmail.com}

\icmlkeywords{Machine Learning, ICML}

\vskip 0.3in
]



\printAffiliationsAndNotice{\icmlEqualContribution} 

\begin{abstract}
Model-agnostic meta learning (MAML) is currently one of the dominating approaches for few-shot meta-learning. Albeit its effectiveness, the optimization of MAML can be challenging due to the innate bilevel problem structure. Specifically, the loss landscape of MAML is much more complex with possibly  more saddle points and local minimizers than its empirical risk minimization counterpart.
To address this challenge, we leverage the recently invented sharpness-aware minimization and develop a sharpness-aware MAML approach that we term Sharp-MAML. We empirically demonstrate that Sharp-MAML and its computation-efficient variant can outperform the plain-vanilla MAML baseline (e.g., $+3\%$ accuracy on Mini-Imagenet). We complement the empirical study with the  convergence rate analysis and the generalization bound of Sharp-MAML. To the best of our knowledge, this is the first empirical and theoretical study on sharpness-aware minimization in the context of bilevel learning. The code is available at \url{https://github.com/mominabbass/Sharp-MAML}.
\end{abstract}

\section{Introduction}

Humans tend to easily learn new concepts using only a handful of samples. In contrast, modern deep neural networks require thousands of samples to train a model that generalizes well to unseen data \cite{Krizhevsky}. Meta learning is a remedy to such a problem whereby new concepts can be learned using a limited number of samples \citep{Schmidhuber, Vilalta}. Meta learning offers fast adaptation to unseen tasks \cite{Thrun, Novak} and has been widely studied to produce state of the art results in a variety of few-shot learning settings including language and vision tasks \cite{Munkhdalai, nichol2018_reptile, Snell, Wang3, Li3, Vinyals, Andrychowicz, Brock,Zintgraf,  Wang5, Achille, Li4, Hsu, Obamuyide}. In particular, model-agnostic meta learning (MAML) is one of the most popular optimization-based meta learning frameworks for few-shot learning \citep{Finn, Vuorio, Yin2, Obamuyide}. MAML aims to learn an initialization  such that after applying only a few number of gradient descent updates on the initialization, the adapted task-specific model can achieve desired  performance on the validation dataset. MAML has been successfully implemented in various data-limited applications including medical image analysis \cite{Maicas}, language modelling \cite{Huang2}, and object detection \cite{Wang4}.

Despite its recent success on some applications, MAML faces a variety of optimization challenges. 
For example, MAML incurs high computation cost due to second-order derivatives, requires searching for multiple hyperparameters, and is sensitive to neural network architectures \cite{Antoniou}. 
Even if various optimization techniques can potentially overcome these training challenges (e.g., making training error small), whether the  meta-model learned with limited training samples can lead to small generalization error or testing error in unseen tasks with unseen data, is not guaranteed~\cite{rothfuss2021_pacoh_pac_bayes_ml}.

These training and generalization challenges of MAML are partially due to the nested (e.g., \emph{bilevel}) structure of the problem, where the upper-level optimization problem learns shared model initialization and the lower-level problem optimizes task-specific models \citep{Finn,Rajeswaran}. 
This is in sharp contrast to the more widely known \emph{single-level} learning framework - empirical risk minimization (ERM). 
As a result, the training and generalization challenges in ERM will not only remain in MAML but may be also exacerbated by the bilevel structure of MAML. For example, as we will show later, the nonconvex loss landscape of MAML contains possibly more saddle points and local minimizers than its ERM counterpart, many of which do not have good generalization performance. Recent works have proposed various useful techniques to improve the generalization performance \cite{Grant, Park, Antoniou, kao2021maml}, but none of them are from the perspective of the optimization landscape.

Given the nested nature of \textit{bilevel} learning models such as MAML, this paper aims to answer the following question:
\begin{tcolorbox}
    \begin{center} 
    {\sf How can we find nonconvex bilevel learning models such as MAML that generalize well?}
    \end{center} 
\end{tcolorbox}

In an attempt to provide a satisfactory answer to this question, we use MAML as a concrete case of bilevel learning and incorporate a recently proposed sharpness-aware minimization (SAM) algorithm \cite{Foret} into the MAML baseline. 
Originally designed for \textit{single-level} problems such as ERM, SAM improves the generalization ability of non-convex models by leveraging the connection between generalization and sharpness of the loss landscape \cite{Foret}. 
We demonstrate the power of integrating SAM into MAML by: i) empirically showing that it outperforms the popular MAML baseline; and, ii) theoretically showing it leads to the potentially improved generalization bound.
To the best of our knowledge, this is the first study on sharpness-aware minimization in the context of bilevel optimization.

\subsection{Our contributions}

We summarize our contributions below.

\textbf{(C1) Sharpness-aware optimization for MAML with improved empirical performance.} 
We theoretically and empirically discover that the loss landscape of bilevel models such as MAML is more involved than its ERM counterpart with possibly more saddle points and local minimizers. To overcome this challenge, we develop a sharpness-aware MAML approach that we term Sharp-MAML and its computation-efficient variant. 
Intuitively, Sharp-MAML avoids the sharp local minima of MAML loss functions and achieves better generalization performance. 
We empirically demonstrate that Sharp-MAML can  outperform the plain-vanilla MAML baseline. 

\textbf{(C2) Optimization analysis of Sharp-MAML including MAML as a special case.} We establish the ${\cal O}(1/\sqrt{T})$ convergence rate of Sharp-MAML through the lens of recent bilevel optimization analysis \cite{Chen}, where $T$ is the number of iterations. This corresponds to ${\cal O}(\epsilon^{-2})$ sample complexity as a fixed number of samples are used per iteration. The convergence rate and sample complexity match those of training single-level ERM models, and improves the known ${\cal O}(\epsilon^{-3})$ sample complexity of MAML.

\textbf{(C3) Generalization analysis of Sharp-MAML demonstrating its improved generalization performance.} We quantify the generalization performance of models learned by Sharp-MAML through the lens of a recently developed probably approximately correct (PAC)-Bayes framework  \cite{farid2021generalization}. The generalization bound justifies the desired empirical performance of models learned from Sharp-MAML, and provides 
some insights on why models learned through Sharp-MAML can have better generalization performance than that from MAML.

\subsection{Technical challenges}
Due to the bilevel structure of both SAM and MAML, formally quantifying the optimization and generalization performance of Sharp-MAML is highly nontrivial.

Specifically, the state-of-the-art convergence analysis of bilevel optimization (e.g., \citep{Chen}) only applies to the case where the upper- and lower-level are both  minimization problems. Unfortunately, this  prerequisite is not satisfied in Sharp-MAML. In addition, the existing analysis of  MAML in \cite{fallah2020convergence} requires the growing batch size and thus results in a suboptimal ${\cal O}(\epsilon^{-3})$ sample complexity. From the theoretical perspective, this work not only broadens the applicability of the recent analysis of bilevel optimization  \citep{Chen} to tackle Sharp-MAML problems, but also tightens the analysis of the original MAML  \cite{fallah2020convergence}. 
For the generalization analysis of Sharp-MAML, different from the classical PAC-Bayes analysis for single-level problems as in SAM~\cite{Foret}, both the lower and upper level problems of MAML contribute to the  generalization error.
Going beyond the PAC-Bayes analysis in~\cite{Foret}, we further discuss how the choice of the perturbation radius {in SAM} affects the bound, providing insights on why Sharp-MAML improves over MAML in terms of generalization ability. 

\subsection{Related work}

We review related work from the following three aspects. 

\textbf{Loss landscape of non-convex optimization.} The connection between the flatness of minima and the generalization performance of the minimizers has been studied both theoretically and empirically; see e.g., \cite{Dziugaite, Dinh, Keshakar, Neyshabur}. In a recent study, \cite{Jiang2} has showed empirically that sharpness-based measure has the highest correlation with generalization. Furthermore, \citep{Izmailov} has showed that averaging model weights during training yields flatter minima that can generalize better.

\textbf{Sharpness-aware minimization.}
Motivated by the connection between sharpness of a minimum and generalization performance,  \cite{Foret} developed the SAM algorithm that encourages the learning algorithm to converge to a flat minimum, thereby improving its generalization performance. Recent follow-up works on SAM showed the efficacy of SAM in various settings. Notably, \cite{Bahri} used SAM to improve the generalization performance of language models like text-to-text Transformer  \cite{Raffel} and its multilingual counterpart \cite{Xue}. More importantly, they empirically showed that the gains achieved by SAM are even more when \emph{the training data are limited}. Furthermore, \citep{Chen2} showed that vision models such as transformers \cite{Dosovitskiy} and MLP-mixers \cite{Tolstikhin} suffer from sharp loss landscapes that can be better trained via SAM. They showed that the generalization performance of resultant models improves across various tasks including supervised, adversarial, contrastive, and transfer learning (e.g., 11.0$\%$ increase in top-1 accuracy). 
However, existing efforts have been focusing on improving generalization performance in \textit{single-level} problems such as ERM \cite{Bahri, Chen2}. 
Different from these works based on single-level ERM, we study SAM in the context of MAML through the lens of bilevel optimization. Recent works aim to reduce the computation overhead of SAM. In \citep{Du2}, two new variants of SAM have been proposed namely, \emph{Stochastic Weight Perturbation} and \emph{Sharpness-sensitive Data Selection}, both of which improve the efficiency of SAM without sacrificing generalization performance. While this work showed remarkable improvement on a standard ERM model, whether it can improve the computation overhead (without sacrificing  generalization) of a MAML-model is unknown.

\textbf{Model-agnostic meta learning.} 
Since it was first developed in \citep{Finn}, MAML has been one of the most popular optimization-based meta learning tools for fast few-shot learning.
Recent studies revealed that the choice of the lower-level optimizer affects the generalization performance of MAML \cite{Grant, Antoniou, Park}. \citep{Antoniou} pointed out a variety of issues of training MAML, such as sensitivity to neural network architectures that leads to instability during training and high computational overhead at both training and inference times. They proposed multiple ways to improve the generalization error, and stabilize training MAML, calling the resulting framework MAML++. Many recent works focus on 
analyzing the generalization ability of MAML~\cite{farid2021generalization,denevi2018_l2l_linear_centroid,rothfuss2021_pacoh_pac_bayes_ml,chen2022bayes_maml} and 
improving the generalization performance of MAML \cite{Finn2, Gonzalez, Park}. 
However, these works do not take into account the geometry of the loss landscape of MAML. In addition to generalization-ability, recent works \citep{Wang, Goldblum, Xu} investigated MAML from another important perspective of \emph{adversarial robustness} the capabilities of a model to defend against adversarial perturbed inputs (also known as adversarial attacks in some literature). However, we focus on improving the generalization performance of the models trained by MAML with theoretical guarantees.

\section{Preliminaries and Motivations}

In this section, we first review the basics of MAML and describe the optimization difficulty of learning MAML models, followed by introducing the SAM method.

 \subsection{{Problem formulation of MAML}} \label{sec2.1}

The goal of few-shot learning is to train a model that can quickly adapt to a new task using only a few datapoints (usually 1-5 samples per task). Consider ${M}$ few-shot learning tasks $\{\mathcal{T}_{m}\}_{m=1}^{M}$ drawn from a distribution $p(\mathcal{T})$. Each task $m$ has a fine-tuning training set $\mathcal{D}_m = \cup^{n}_{i=1}\{(x_i,y_i)\}$ and a separate validation set $\mathcal{D}_m^{\prime}=\cup^{n}_{i=1}\{(x_i,y_i)\}$, where data are independently and identically distributed (i.i.d.) and drawn from the per-task data distribution $\mathcal{P}_m$.
MAML seeks to learn a good initialization of the model parameter $\theta$ (called the meta-model) such that fine-tuning $\theta$ via a small number of gradient updates will lead to fast learning on a new task. Consider a per datum loss $l:\Theta\times \mathcal{X} \times \mathcal{Y} \rightarrow \mathbb{R}_{+}$; define the generic \emph{empirical loss} over a finite-sample dataset $\mathcal{D}$ as $\mathcal{L} (\theta;\mathcal{D}) = \frac{1}{n} \sum_{i=1}^{n} l(\theta, x_i, y_i)$ and the generic \emph{population loss} over a data distribution $\mathcal{P}$ as $\mathcal{L} (\theta;{\mathcal{P}})=\mathbb{E}_{(x, y)\sim \mathcal{P}}[l(\theta, x, y)]$. For a particular task $m$, they become  $\mathcal{L} (\theta;\mathcal{D}_m)$ or $\mathcal{L} (\theta;\mathcal{D}_m')$ and $\mathcal{L} (\theta;{\mathcal{P}_m})$.

MAML can be formulated as a bilevel optimization problem, where the fine-tuning stage forms a task-specific lower-level problem while the meta-model $\theta$ optimization forms a shared upper-level problem. 
Namely, the optimization problem of MAML is \citep{Rajeswaran}: 
\begin{subequations}  \label{MAML}
\begin{align} 
  \!\!  \min_{\theta} &~~\frac{1}{M} \sum_{m=1}^{M} \mathcal{L} (\theta_{m}^{*}(\theta); \mathcal{D}_m^{\prime})  \label{MAMLa} \\
  \!\!     {\rm s. t.} &\,\,\, \theta_{m}^{*}(\theta) \!=\! \arg \min_{\theta_{m}} \mathcal{L}  (\theta_{m}; \mathcal{D}_m) \!+\! \frac{\|\theta_m-\theta\|^2}{2\beta_{\rm low}},\, \forall{m}  \!  \label{MAMLb}
\end{align}
\end{subequations}
where $\beta_{\rm low}$ denotes the lower-level step size.

The bilevel optimization problem in \eqref{MAML} is difficult to solve because each upper-level update \eqref{MAMLa} requires calling lower optimization oracle multiple times \eqref{MAMLb}. There exist many MAML algorithms to solve \eqref{MAML} efficiently, such as Reptile \citep{nichol2018_reptile} and first-order MAML \cite{Finn} which is an approximation to MAML obtained by ignoring second-order derivatives. We instead use the one-step gradient update \cite{Finn} to approximate the lower-level problem: 
\begin{equation} \label{eq2}
\min_{\theta}  ~F(\theta) \triangleq {\rm \eqref{MAMLa}} ~~~{\rm s. t.}~ \theta_{m}^{\prime}(\theta) = \theta - \beta_{\rm low} \nabla \mathcal{L} (\theta; \mathcal{D}_m).
\end{equation}

\begin{figure}[t]
    \centering
    \includegraphics[width=.5\textwidth]{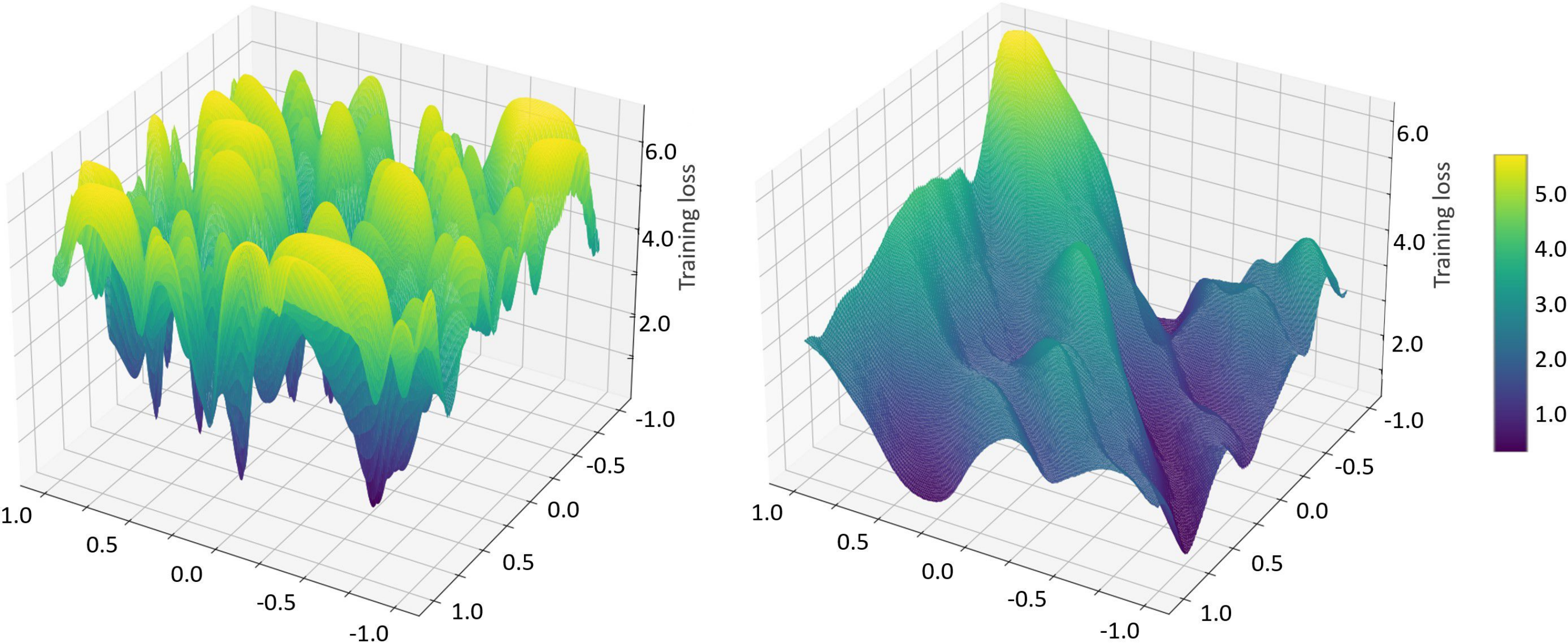}
    \vspace{-0.4cm}
    \caption{Loss landscapes of MAML and ERM for a single task on CIFAR-100 dataset. Details of the architecture are given in Section \ref{sec5.1}. 
We use the cross-entropy loss following the process in \citep{Li9}. \textbf{Left:} Loss landscape of a MAML model (5-way 1-shot). \textbf{Right:} Loss landscape of a standard ERM model.}
    \vspace{-0.2cm}
    \label{figure1}
\end{figure}

\textbf{Generalization performance.} 
We are particularly interested in the generalization performance of a meta-model $\theta$ obtained by solving the empirical MAML problem \eqref{eq2}. 
The generalization performance of a meta-model $\theta$ is measured by the expected population loss 
\begin{align}  \label{MAMLc}
\!\mathcal{L}(\theta_m(\theta); \mathcal{P})\triangleq\mathbb{E}_{\mathcal{T}_{m}}\mathbb{E}_{(x,y)\sim \mathcal{P}_m}[\mathcal{L} (\theta_m(\theta;\mathcal{D}_m); \mathcal{D}_m')]
\end{align}
where the expectation is taken over the randomness of the  sampled tasks as well as data in the training and validation datasets per sampled task. 
For notation simplicity, we define the marginal data distribution for variable $(x,y)$ as 
\begin{align}
    \mathcal{P}(x,y) \triangleq\mathbb{E}_{\mathcal{T}_m}[\mathcal{P}_m(x,y)]= \int P(x,y\mid \mathcal{T}_m) P(\mathcal{T}_m) d \mathcal{T}_m
\end{align}
and we use $\mathcal{P}$ to represent $\mathcal{P}(x,y)$ thereafter.
 
 \subsection{Local minimizer of ERM implies that of MAML}
 Nevertheless, even with the simple lower-level gradient descent step \eqref{eq2}, training the upper-level meta-model $\theta$ still requires differentiating the lower update. In other words, the meta-model requires the second-order information (i.e., the Hessian) of the objective function with respect to $\theta$, making the problem \eqref{MAML} more involved than an ERM formulation of the multi-task learning (called ERM thereafter), given by
\begin{equation} \label{ERM}
    \min_{\theta}~ \frac{1}{M} \sum_{m=1}^{M} \mathcal{L} (\theta; \mathcal{D}_m).
\end{equation}
To understand the difficulty of the MAML objective in \eqref{eq2}, we visualize its loss landscape of a particular task $m$ and compare it with that of ERM for the same task $m$. Figure \ref{figure1} shows the per-task loss landscapes of a meta-model $\theta$ in \eqref{MAMLa} and a standard ERM model in \eqref{ERM} on CIFAR-100 dataset. We find that the loss landscape of a meta-model is indeed much more involved with more local minima, making the optimization problem difficult to solve. The following lemma also characterizes the complex landscape of meta-model on a particular task $m$ and its proof is deferred in Appendix \ref{pf:lm1}.

\begin{lemma}[Local minimizers of MAML]\label{maml-infer}
Consider the one-step gradient fine-tuning step $\eqref{eq2}$.
For any $m\in\mathcal{M}$, assume $\mathcal{L}(\theta;\mathcal{D}_m)$ has continuous third-order derivatives. Then for a particular task $m$, the following two statements hold\\
a) the stationary points for $\mathcal{L}(\theta;\mathcal{D}_m)$ are also the stationary points for $\mathcal{L}(\theta_m^\prime(\theta);\mathcal{D}_m)$; and,\\
b) the local minimizers for $\mathcal{L}(\theta;\mathcal{D}_m)$ are also the local minimizers for $\mathcal{L}(\theta_m^{\prime}(\theta);\mathcal{D}_m)$. 
\end{lemma}

Lemma \ref{maml-infer} shows that for a given task $m$, the number of stationary points and local minimizers for ERM's loss
$\mathcal{L}(\theta;\mathcal{D}_m)$ are fewer than that of MAML's loss $\mathcal{L}(\theta_m^{\prime}(\theta);\mathcal{D}_m)$, which is aligned with the empirical observations in Figure \ref{figure1}. 
While some of the local minimizers in MAML's loss landscape are indeed effective few-shot learners, there are a number of sharp local minimizers in MAML that may have undesired generalization performance.
It also suggests that the optimization of MAML can be more challenging than its ERM counterpart.

   

\subsection{Sharpness aware minimization}
SAM is a recently developed technique that leverages the geometry of the loss landscape to improve the generalization performance by simultaneously minimizing the loss value and the loss sharpness \cite{Foret}. Given the empirical loss $\mathcal{L}(\theta; \mathcal{D})$, the goal of training is to choose $\theta$ having low population loss $\mathcal{L} (\theta;{\mathcal{P}})$. SAM achieves this through the following optimization problem
\begin{equation}  \label{sam}
    \min_{\theta}~ \mathcal{L}^{\rm sam}(\theta;\mathcal{D}) ~{\rm with}~ \mathcal{L}^{\rm sam}(\theta;\mathcal{D}) \!\triangleq\! \max_{||\epsilon||_{2} \leq \alpha}  \mathcal{L}(\theta + \epsilon;\mathcal{D}).
\end{equation}
Given $\theta$, the maximization in \eqref{sam} seeks to find the weight perturbation $\epsilon$ in the Euclidean ball with radius $\alpha$ that maximizes the empirical loss. If we define the \textbf{sharpness} as
\begin{equation}
    \max_{||\epsilon||_2 \leq \alpha}~ [\mathcal{L}(\theta + \epsilon;\mathcal{D})- \mathcal{L}(\theta;\mathcal{D})]
\end{equation}
then \eqref{sam} essentially minimizes the sum 
of the sharpness and the empirical loss $\mathcal{L}(\theta;\mathcal{D})$. 
While the maximization in \eqref{sam} is generally costly, a closed-form approximate maximizer has been proposed in \cite{Foret} by invoking the Taylor expansion of the empirical loss.  
In such case, SAM seeks to find a flat minimum  by iteratively applying the following two-step procedure at each iteration $t$, that is
\begin{subequations}  \label{eq5}
\begin{align}
&{\epsilon}(\theta^{t}) = \alpha{\nabla \mathcal{L}(\theta^{t};\mathcal{D})}/{||\nabla \mathcal{L}(\theta^{t};\mathcal{D})||_2}  \label{eq5b}\\
	&\theta^{t+1} = \theta^{t} - \beta^{t} (\nabla \mathcal{L}(\theta^{t} + {\epsilon}(\theta^{t});\mathcal{D}))\label{eq5a}
\end{align}
\end{subequations}
where $\beta^{t}$ is an appropriately scheduled learning rate. 
In \eqref{eq5a} and thereafter, the notation $\nabla \mathcal{L}(\theta+\epsilon_m(\theta))$ means 
$\nabla \mathcal{L}(\theta+\epsilon_m(\theta))\triangleq\nabla_x \mathcal{L}(x) |_{ x=\theta+\epsilon_m(\theta)}$.
SAM works particularly well for complex and non-convex problems having a myriad of local minima, and where different minima yield models with different generalization abilities.

\section{Sharp-MAML: Sharpness-aware MAML} \label{sec3}

As discussed in Section \ref{sec2.1}, MAML has a complex loss landscape with multiple local and global minima, that may yield similar values of empirical loss $\mathcal{L}(\theta;\mathcal{D})$ while having significantly different generalization performance. Therefore, we propose integrating SAM with MAML, which is a new bilevel optimization problem.

\subsection{Problem formulation of Sharp-MAML}
We propose a unified optimization framework for Sharpness-aware MAML that we term Sharp-MAML by using two hyperparameters $\alpha_{\rm up} \geq 0$ and  $\alpha_{\rm low} \geq 0$, that is:
\begin{tcolorbox}
\leqnomode
\vspace{-0.3cm}
\begin{align}\label{SAM_MAML}
\hspace{-2cm}\tag{\bf P}\hspace{-1cm}& ~~~~~~~ \min_{\theta} \textcolor{blue}{\max_{||\epsilon||_2 \leq \alpha_{\rm up}}}  \sum_{m=1}^{M} \mathcal{L}  (\theta_{m}^{*}(\theta+\textcolor{blue}{\epsilon}); \mathcal{D}_m^{\prime})~~~\, \text{\sf (upper)}\hspace{-2cm}\\
&{\rm {\rm s. t.}}~~\theta_{m}^{*}(\theta) = \arg \min_{\theta_m}  \textcolor{red}{\max_{||\epsilon_m||_2 \leq \alpha_{\rm low}}} \!\mathcal{L} (\theta_{m}+\textcolor{red}{\epsilon_m}; \mathcal{D}_m)\nonumber\\
&~~~~~~~~~~~~~~~~ ~+ \frac{\|\theta_m-\theta\|^2}{2\beta_{\rm low}},~  {m}=1,...,M.~~~~~~\text{\sf (lower)} \nonumber
\end{align}
\reqnomode
\end{tcolorbox}
Compared with the bilevel formulation of MAML in \eqref{MAML}, the above Sharp-MAML formulation is a four-level problem. However, in our algorithm design, we will efficiently approximate the two maximizations in \eqref{SAM_MAML} so that the cost of Sharp-MAML is almost the same as that of MAML.

In what follows, we list three main technical questions that we aim to address.

\textbf{Q1.} The choice of $\alpha_{\rm up}$, $\alpha_{\rm low}$ determines the specific scenario of integrating SAM with MAML. Applying SAM to both fine-tuning and meta-update stages would be computationally very expensive. Spurred by that, we ask: \emph{Is it possible to achieve better generalization by incorporating SAM into only either upper- or lower-level problem?}

\textbf{Q2.} Both MAML in \eqref{MAML} and SAM in \eqref{sam} are bilevel optimization problems requiring several lower optimization steps. Thus, we also study whether or not the computationally-efficient alternatives (e.g. ESAM \cite{Du2}, ANIL \citep{Raghu}) can promise good generalization.

\textbf{Q3.} The theoretical motivation for SAM has been illustrated in \citep{Foret} by bounding generalization ability in terms of neighborhood-wise training loss. Spurred by that, we further ask: \emph{Can we explain {and theoretically justify} why integrating SAM with MAML is effective in promoting generalization performance of MAML models?}

\subsection{Algorithm development} \label{sec3.2}
Based on \eqref{SAM_MAML}, we focus on three variants of Sharp-MAML that differ in their respective computational complexity:

{\bf (a) Sharp-MAML$_{\rm low}$}: SAM is applied to only the fine-tuning step, i.e., $\alpha_{\rm low} > 0$ and $\alpha_{\rm up}=0$.\\
{\bf (b) Sharp-MAML$_{\rm up}$}: SAM is applied to only the meta-update step, i.e., $\alpha_{\rm low} = 0$ and $\alpha_{\rm up}>0$.\\
{\bf (c) Sharp-MAML$_{\rm both}$}: SAM is applied to both fine-tuning and meta-update steps, i.e., $\alpha_{\rm up}$, $\alpha_{\rm low} > 0$.

Below we only introduce Sharp-MAML$_{\rm both}$ in detail and leave the pseudo-code of the other two variants in Appendix \ref{appendix::A} since the other two variants can be deduced from Sharp-MAML$_{\rm both}$. 
For the sake of convenience, we define the biased mini-batch gradient descent (BGD) at point $\theta^t+\epsilon$ using gradient at $\theta^t+\epsilon+\epsilon_m$  as
\begin{equation}  \label{perturb}
    {\rm BGD}_m(\theta^t,\epsilon,\epsilon_m) \triangleq \theta^t+\epsilon - \beta_{\rm low} \widetilde\nabla \mathcal{L} ({\theta^t+\epsilon+\epsilon_m}; \mathcal{D}_{m})
\end{equation}
where $\epsilon$ and $\epsilon_m$ are perturbation vectors that can be computed accordingly to different Sharp-MAML, and $\widetilde\nabla\mathcal{L}(~\cdot~;\mathcal{D}_m)$ is an unbiased estimator of $\nabla\mathcal{L}(~\cdot~;\mathcal{D}_m)$ which can be assessed by mini-batch evaluation. Moreover, letting $\tilde\theta_m(\theta^t)={\rm BGD}_m(\theta^t,\epsilon,\epsilon_m)$, we define 
\begin{align}
  \!\!  &  \!\! \nabla_{\theta_t}\mathcal{L}(\tilde\theta_m(\theta^t);\mathcal{D}_m^\prime)\triangleq\nonumber\\
  \!\!     &~(I\!-\!\beta_{\rm low}\nabla^2\mathcal{L}(\theta^t+\epsilon\!+\!\epsilon_m;\mathcal{D}_m))\nabla\mathcal{L}(\tilde\theta_m(\theta^t);\mathcal{D}_m^\prime) 
\end{align}
and $\nabla^2\mathcal{L}(\theta^t+\epsilon+\epsilon_m;\mathcal{D}_m)$ is the Hessian matrix of $\mathcal{L}(~\cdot~;\mathcal{D}_m)$ at $\theta^t+\epsilon+\epsilon_m$. 

\paragraph{Sharp-MAML$_{\rm both}$.} For each task $m$, we compute its corresponding perturbation ${\epsilon_{m}}(\theta^t)$ as follows:
 \begin{equation}  \label{eq7}
  \epsilon_m (\theta^t) = \alpha_{\rm low}{\widetilde\nabla \mathcal{L} (\theta^t; \mathcal{D}_m)}/{||\widetilde\nabla \mathcal{L} (\theta^t; \mathcal{D}_m)||_2}.
\end{equation}

Thereafter, the fine-tuning step is carried out by performing gradient descent at $\theta^t$ using the gradient at the maximum point $\theta^t+ {\epsilon_{m}} (\theta^t)$ using \eqref{perturb}:
\begin{equation}  \label{eq8}
    \tilde\theta_{m}^1(\theta^t) = {\rm BGD}_m(\theta^t,0,\epsilon_m(\theta^t)).
\end{equation}

 \begin{algorithm}[htb]
    \normalsize
	\caption{Pseudo-code for {\bf Sharp-MAML$_{\rm both}$}; \colorbox{red!30}{red lines} need to be modified for \red{\bf Sharp-MAML$_{\rm up}$}; \colorbox{blue!30}{blue lines} need to be modified for \blue{\bf Sharp-MAML$_{\rm low}$}}
	\label{alg4}
	\begin{algorithmic}[1]
	\State {\bf Require:} $p( \mathcal{T})$: distribution over tasks 
	\State {\bf Require:} $\beta_{\rm low}, \beta_{\rm up}$:   step sizes
	\State {\bf Require:} $\alpha_{\rm low}>0, \alpha_{\rm up} >0$: perturbation radii
    	\For{$t=1,\cdots, T$}
            \State Sample batch of tasks $\mathcal{T}_{m} \sim p(\mathcal{T})$
            \For{{\bf all} $\mathcal{T}_{m}$}
            \State Sample $K$ examples from  $\mathcal{D}_{m}$
            \State Evaluate $\widetilde\nabla \mathcal{L}({\theta^t}; \mathcal{D}_{m})$ 
            \State \colorbox{red!30}{Compute perturbation ${\epsilon_m}(\theta^t)$ via \eqref{eq7}}
            \State \colorbox{red!30}{Compute fine-tuned parameter $\tilde\theta_{m}^1(\theta^t)$ via \eqref{eq8}}
            \State Sample data from $\mathcal{D}_m^{\prime}$ for meta-update
            \EndFor
        \State \colorbox{blue!30}{Compute $\sum_{m=1}^{M} \widetilde \nabla \mathcal{L}(\tilde\theta_{m}^1(\theta^t);\mathcal{D}_m^{\prime})$ }
        \State \colorbox{blue!30}{Compute perturbation ${\epsilon}(\theta^t)$ via \eqref{eq9}}
        \State \colorbox{blue!30}{Update $\theta$ via \eqref{eq12} using $\hat\theta_{m}^{2}(\theta^t)$ from \eqref{eq10}}
        \EndFor
    \end{algorithmic} 
\end{algorithm}

After we obtain $\tilde\theta_m^1(\theta^t)$ for all tasks, we compute the mini-batch gradient estimator of the upper loss i.e., $\nabla h = \widetilde\nabla_{\theta^t} \sum_{m=1}^{M} \mathcal{L}(\tilde\theta_{m}^1(\theta^t);\mathcal{D}_m^{\prime})$ which is an unbiased estimator of the upper-level gradient  $\nabla_{\theta^t} \sum_{m=1}^{M} \mathcal{L}(\tilde\theta_{m}^1(\theta^t);\mathcal{D}_m^{\prime})$, and use it to compute the meta perturbation $\epsilon (\theta^t)$ by:
\begin{equation}  \label{eq9}
  \epsilon (\theta^t) =  \alpha_{\rm up}{\nabla h}/{||\nabla h||_2}.
\end{equation}

Afterwards, we compute the new perturbed fine-tuned parameter, denoted by $\tilde\theta_{m}^2(\theta^t)$, using the gradient at the maximum point $\theta^t+\epsilon(\theta^t)+\epsilon_m(\theta^t)$ in \eqref{perturb}, that is:
\begin{equation}  \label{eq10}
    \tilde\theta_{m}^2(\theta^t) = {\rm BGD}_m(\theta^t,\epsilon(\theta^t),\epsilon_m(\theta^t)).
\end{equation}
Finally, for the meta-update stage, we evaluate the upper loss using the fine-tuned parameter $\tilde\theta_{m}^{2}(\theta^t)$ obtained from \eqref{eq10} and update the meta-parameter $\theta$ via:
\begin{equation}  \label{eq12}
    \theta^{t+1}= \theta^t - \beta_{\rm up} \sum_{m=1}^{M} \widetilde\nabla_{\theta^t} \mathcal{L}(\tilde\theta_{m}^2(\theta^t);\mathcal{D}_m^{\prime}).
\end{equation}

See the pseudocode of Sharp-MAML$_{\rm both}$ in Algorithm \ref{alg4}. The algorithms for Sharp-MAML$_{\rm up}$ and Sharp-MAML$_{\rm low}$ can be deduced by setting $\epsilon_m(\theta_t) = 0$ and $\epsilon(\theta_t) = 0$ in Algorithm \ref{alg4}, respectively, which are formally stated in Algorithm \ref{alg3} and Algorithm \ref{alg1} in Appendix \ref{appendix::A}.

\section{Theoretical Analysis of Sharp-MAML}
In this part, we rigorously analyze the performance of the proposed Sharp-MAML method in terms of the convergence rate and the generalization error.

\subsection{Optimization analysis}
To quantify the optimization performance of solving the one-step version of \eqref{MAML}, we introduce the following  assumptions.

\begin{assumption}[Lipschitz continuity]\label{li-lip}
Assume that $\mathcal{L}(\theta;\mathcal{D}_m^\prime),\nabla\mathcal{L}(\theta;\mathcal{D}_m),\nabla\mathcal{L}(\theta;\mathcal{D}_m^\prime),\nabla^2\mathcal{L}(\theta;\mathcal{D}_m)$, $\forall m$ are Lipschitz continuous with constant $\ell_0,\ell_1,\ell_1,\ell_2$.
\end{assumption}

\begin{assumption}[Stochastic derivatives]\label{variance}
Assume that $\widetilde\nabla\mathcal{L}(\theta;\mathcal{D}_m),\widetilde\nabla^2\mathcal{L}(\theta;\mathcal{D}_m),\widetilde\nabla\mathcal{L}(\theta;\mathcal{D}_m^\prime)$ are unbiased estimator of $\nabla\mathcal{L}(\theta;\mathcal{D}_m),\nabla^2\mathcal{L}(\theta;\mathcal{D}_m),\nabla\mathcal{L}(\theta;\mathcal{D}_m^\prime)$ respectively and their variances are bounded by $\sigma^2$.
\end{assumption}
Assumptions \ref{li-lip}--\ref{variance} also appear similarly in the convergence analysis of meta learning and bilevel optimization \cite{finn2019online,Rajeswaran,fallah2020convergence,Chen,ji2022theoretical,chen2021stable}. 

With the above assumptions, we introduce a novel biased MAML framework which includes MAML and sharp-MAML as special cases and get the following theorem. The proof is deferred in Appendix C. 
\begin{theorem}\label{smaml-con}
Under Assumption \ref{li-lip}--\ref{variance}, and choosing stepsizes and perturbation radii $\beta_{\rm low},\beta_{\rm up},\alpha_{\rm up}={\cal O}(\frac{1}{\sqrt{T}}),\alpha_{\rm low}={\cal O}(1)$,  with some proper constants, we can get that the iterates $\{\theta^t\}$ generated by Sharp-MAML$_{\rm up}$, Sharp-MAML$_{\rm low}$ and Sharp-MAML$_{\rm both}$ satisfy
    \begin{equation}
        \frac{1}{T}\sum_{t=1}^T\mathbb{E}\left[\|\nabla F(\theta^t)\|^2\right]={\cal O}\left(\frac{1}{\sqrt{T}}\right)
    \end{equation}
    where $F(\theta)$ is the objective function of MAML in \eqref{eq2}.
\end{theorem}

Theorem 1 implies that by choosing a proper perturbation threshold, all three versions of Sharp-MAML can still find $\epsilon$ stationary points for MAML objective \eqref{eq2} with ${\cal O}(\epsilon^{-2})$ iterations and ${\cal O}(\epsilon^{-2})$ samples, which matches or even improves the state-of-the-art sample complexity of MAML \cite{Rajeswaran,fallah2020convergence,ji2022theoretical}. 

\subsection{Generalization analysis}

To analyze the generalization error of Sharp-MAML, we make similar assumptions to Theorem 2 in~\cite{Foret}. 
Recall the \emph{population loss} $\mathcal{L} (\theta;{\mathcal{P}})=\mathbb{E}_{(x, y)\sim \mathcal{P}}[l(\theta, x, y)] $.
Denote the stationary point obtained by Sharp-MAML$_{\rm up}$ algorithm as $\hat{\theta}$. 
Note that the Sharp-MAML adopts  gradient descent (GD) as the lower level algorithm, which is uniformly stable based on Definition 1.

\begin{defi}[\cite{hardt2016train}]
An  algorithm $A$ is $\gamma$-uniformly stable if for all data sets $S, S^{\prime} \in Z^{n}$ such that $S$ and $S^{\prime}$ differ in at most one example, we have
\begin{equation}
  \sup_S \left|\mathbb{E}_{S}\left[l\left(A(S) ; x,y\right)-l\left(A\left(S^{\prime}\right) ; x,y\right)\right]\right| \leq \gamma
\end{equation}
where $A(S)$ and $A(S^{\prime})$ are the outputs of the algorithm $A$ given datasets $S$ and $S^{\prime}$.
\end{defi}

With the above definition of uniform stability, we are ready to establish the generalization performance. We defer the proof of Theorem \ref{thm:pac_bayes} to Appendix D.

\begin{theorem}
\label{thm:pac_bayes}
    Assume loss function $\mathcal{L}(\cdot)$ is bounded: $0 \leq \mathcal{L}(\theta_m'; \mathcal{D}) \leq 1$, for  $\theta_m'$ defined in~\eqref{eq2}, and any $\mathcal{D}$. 
    Define 
    $F(\theta ; \mathcal{P})=\mathbb{E}_{\mathcal{D} \sim \mathcal{P}}\left[F(\theta ; \mathcal{D})\right]$.
    Assume 
    $F(\hat{\theta} ; \mathcal{P}) \leq \mathbb{E}_{\epsilon \sim \mathcal{N}\left(0, \alpha^{2} \mathbf{I}\right)}\left[F(\hat{\theta}+\epsilon ; \mathcal{P})\right]$
    at the stationary point of the Sharp-MAML$_{\rm up}$ denoted by  $\hat{\theta}$.
    For parameter $\theta_m'(\hat{\theta}; \mathcal{D})$ learned with $\gamma_{\mathrm{A}}$ uniformly stable algorithm $A$ 
    from $\hat{\theta} \in \mathbb{R}^k$, 
with probability $1-\delta$ over the choice of the training set $\mathcal{D} \sim \mathcal{P}$, with $|\mathcal{D}| = nM$, it holds that
\begin{align}\label{eq:thm_pac_bayes_ml_main}
&F (\hat{\theta};{\mathcal{P}}) \leq 
\max _{\|\epsilon\|_{2} \leq \alpha} F(\hat{\theta}+{\epsilon}; \mathcal{D})
+ \gamma_{\rm A}+\\
&\sqrt{\frac{ k \ln \Big(1+\frac{\|\hat{\theta}\|_{2}^{2}}{ \alpha^{2}} \Big(1+\sqrt{\frac{\ln (nM)}{k}}\Big)^{2} \Big) + 2\ln \frac{1}{\delta} + 5 \ln (nM)}{4nM}}.   \nonumber
\end{align}
\end{theorem}

\textbf{Improved upper bound of generalization error.}
Theorem~\ref{thm:pac_bayes} shows that the difference between the population loss and the empirical loss of Sharp-MAML$_{\rm up}$ is bounded by the stability of the lower-level update $\gamma_A$ and another $\tilde{\mathcal{O}}(k/nM)$ term that vanishes as the number of meta-training data goes to infinity.
The lower-level update GD has uniform stability of order $\mathcal{O}(1/n)$~\cite{hardt2016train}.
Also, it is worth noting that the upper bound of the population loss on the right-hand side (RHS) of \eqref{eq:thm_pac_bayes_ml_main}, is a function of $\alpha$. And for any sufficiently small $\alpha_0>0$, we can find some $\alpha_1 > \alpha_0$, where this function takes smaller value than at $\alpha_0$ (see the proof in Appendix D). This suggests that a choice of $\alpha$ arbitrarily close to zero, in which case Sharp-MAML  reduces to the original MAML method, is not optimal {in terms of the generalization error upper bound}.
Therefore, it shows Sharp-MAML has smaller generalization error upper bound than conventional MAML.
The analysis can be extended to Sharp-MAML$_{\rm both}$ in a similar way.
\vspace{-0.1cm}

\section{Numerical Results}
In this section, we demonstrate the effectiveness of Sharp-MAML by comparing it with several popular MAML baselines in terms of generalization and computation cost. We evaluate Sharp-MAML on 5-way 1-shot and 5-way 5-shot settings on the Mini-Imagenet dataset and present the results on Omniglot dataset in Appendix \ref{appendixE}. 

\subsection{Experiment setups}  \label{sec5.1}
Our model follows the same architecture used by \cite{Vinyals}, comprising of 4 modules with a $3\times3$ convolutions with 64 filters followed by batch normalization \cite{Ioffe}, a ReLU non-linearity, and a $2\times2$ max-pooling. We follow the experimental protocol in \cite{Finn}. The models were trained using the SAM\footnote{We used the open-source SAM PyTorch implementation available at \url{https://github.com/davda54/sam}} algorithm with Adam as the base optimizer and learning rate $\alpha = 0.001$. Following \cite{Ravi}, 15 examples per class were used to evaluate the post-update meta-gradient. The values of $\alpha_{\rm low}$, $\alpha_{\rm up}$ are taken from a set of $\{0.05, 0.005, 0.0005, 0.00005\}$ and each experiment is run on each value for three random seeds. We choose the inner gradient steps from a set of $\{3,5,7,10\}$. The step size is chosen from a set of $\{0.1, 0.01, 0.001\}$. For Sharp-MAML$_{\rm both}$ we use the same value of $\alpha_{\rm low}$, $\alpha_{\rm up}$ in each experiment. We report the best results in Tables \ref{table1} and \ref{table2}. 

One Sharp-MAML update executes two backpropagation operations (i.e., one to compute $\epsilon(\theta)$ and another to compute the final gradient). Therefore, for a fair comparison, we execute each MAML training run twice as many epochs as each Sharp-MAML training run and report the best score achieved by each MAML training run across either the standard epoch count or the doubled epoch count.

\subsection{Sharp-MAML versus MAML baselines} \label{sec5.2}
As baselines, we use MAML \citep{Finn}, Matching Nets \citep{Vinyals}, CAVIA \citep{zintgraf2019_cavia}, Reptile \citep{nichol2018_reptile}, FOMAML \citep{nichol2018_reptile}, and BMAML \citep{yoon2018_BMAML}. 

\begin{table}[t]
\vskip -0.1in
\caption{Results on Mini-Imagenet (5-way 1-shot). Our reproduced result of MAML is close to that of the original$^*$.
}
\label{table1}
\begin{center}
\begin{small}
\begin{sc}
\begin{tabular}{lcccr}
\toprule
Algorithms  & Accuracy\\
\midrule
Matching Nets                         & 43.56$\%$\\
iMAML~\cite{rajeswaran2019_imaml}
  & 49.30 $\%$\\
CAVIA~\cite{zintgraf2019_cavia}
  & 47.24 $\%$ \\
Reptile~\cite{nichol2018_reptile} 
  & 49.97 $\%$ \\
FOMAML~\cite{nichol2018_reptile} 
  & 48.07 $\%$ \\
LLAMA~\cite{grant2018recasting}
  & 49.40  $\%$  \\
BMAML~\cite{yoon2018_BMAML}
  & 49.17  $\%$  \\
MAML (reproduced)                           & 47.13$\%$  \\
Sharp-MAML$_{\rm low}$                        & 49.72$\%$  \\
Sharp-MAML$_{\rm up}$                        & 49.56$\%$ \\
Sharp-MAML$_{\rm both}$                         &  {50.28$\%$  }\\

\bottomrule
\end{tabular}
\end{sc}
\end{small}
\end{center}
\vskip -0.05in
\footnotesize{$^*$ reproduced using the Torchmeta \cite{deleu2019torchmeta} library}
\end{table}

\begin{table}[t]
\caption{Results on Mini-Imagenet (5-way 5-shot). Our reproduced result of MAML is close to that of the original$^*$.
}
\label{table2}
\begin{center}
\begin{small}
\begin{sc}
\begin{tabular}{lcccr}
\toprule
Algorithms  & Accuracy\\
\midrule
Matching Nets                         & 55.31$\%$\\
CAVIA~\cite{zintgraf2019_cavia}
  & 59.05$\%$  \\
Reptile~\cite{nichol2018_reptile} &65.99$\%$ \\
FOMAML~\cite{nichol2018_reptile} 
  &63.15 $\%$ \\
BMAML~\cite{yoon2018_BMAML}
  &64.23 $\%$ \\
MAML (reproduced)                           & 62.20$\%$  \\
Sharp-MAML$_{\rm low}$                        & 63.18$\%$  \\
Sharp-MAML$_{\rm up}$                        & 63.06$\%$ \\
Sharp-MAML$_{\rm both}$                         &  {65.04$\%$  }\\

\bottomrule
\end{tabular}
\end{sc}
\end{small}
\end{center}
\vskip -0.05in
\footnotesize{$^*$ reproduced using the Torchmeta \cite{deleu2019torchmeta} library}
\vskip -0.2in
\end{table}

In Tables \ref{table1} and \ref{table2}, we report the accuracy of three variants of Sharp-MAML and other baselines on the Mini-Imagenet dataset in the 5-way 1-shot and 5-way 5-shot settings respectively. We observe that Sharp-MAML$_{\rm low}$ outperforms MAML in all cases, exhibiting the advantage of our methods.

The results on the Omniglot dataset are reported in Table \ref{table3} and Table \ref{table4} of Appendix. Our results verify the efficacy of all the three variants of Sharp-MAML, suggesting that SAM indeed improves the generalization performance of bi-level models like MAML by seeking out flatter minima. 

Since Sharp-MAML requires one more gradient computation per iteration than the original MAML, for a fair comparison, we report the execution times in Table \ref{table5}. The results show that Sharp-MAML$_{\rm low}$ requires the least amount of additional computation while still achieving significant performance gains. Sharp-MAML$_{\rm up}$ and Sharp-MAML$_{\rm both}$  also improves the performance significantly but both approaches have a higher computation than Sharp-MAML$_{\rm low}$ since the computation of additional Hessians is needed for the meta-update gradient. 
\vspace{-0.1cm}
 \begin{table}[t]
\caption{Results on Mini-Imagenet (5-way 1-shot). }
\label{table5}
\vskip 0.05in
\begin{center}
\begin{small}
\begin{sc}
\begin{tabular}{lcccr}
\toprule
 Algorithms   & Accuracy & Time$^\dagger$\\
\midrule
MAML (reproduced)                                & 47.13$\%$ & x1 \\
Sharp-MAML$_{\rm low}$                          & 49.72$\%$ & x2.60\\
Sharp-MAML$_{\rm up}$                          & 49.56$\%$ & x3.60\\
Sharp-MAML$_{\rm both}$                          & 50.28$\%$ & x4.60\\
Sharp-MAML$_{\rm low}$ -ANIL                   & 49.19$\%$ & x1.40\\
ESharp-MAML$_{\rm low}$                         & 48.90$\%$ & x2.20\\
ESharp-MAML$_{\rm low}$ -ANIL                   & 49.03$\%$ & x1.20\\
\bottomrule
\end{tabular}
\end{sc}
\end{small}
\end{center}
\vskip -0.05in
\footnotesize{$^\dagger$ execution time is normalized to MAML training time}
\vskip -0.1in
\end{table}
\subsection{Ablation study and loss landscape visualization} 
We conduct an ablation study on the effect of perturbation radii $\alpha_{\rm low}$ and $\alpha_{\rm up}$ on the three Sharp-MAML variants. Figure \ref{figure2} and Figure \ref{figure3} summarize the results on the Mini-Imagenet dataset. We observe that all the three Sharp-MAML variants outperform the original MAML for almost all the values of $\alpha_{\rm low}$, $\alpha_{\rm up}$ we used in our experiments. Therefore, integrating SAM at any/both stage(s) gives better performance than the original SAM for a wide range of values of the perturbation sizes, reducing the need to fine-tune these hyperparameters. This also suggests that SAM is effectively avoiding bad local minimum in MAML loss landscape (cf. Figure \ref{figure1}) for a wide range of perturbation radii. In Figure \ref{figure4}, we plot the loss landscapes of MAML and Sharp-MAML, and observe that Sharp-MAML indeed seeks out landscapes that are smoother as compared to the landscape of original MAML and hence, meets our theoretical characterization of improved generalization performance. Furthermore, the generalization error of Sharp-MAML$_{\rm both}$ is found to be 34.58$\%$/8.56$\%$ as compared to 37.46$\%$/11.58$\%$ of MAML for the 5-way 1-shot/5-shot Mini-Imagenet, which explains the advantage of our approach.
 \begin{figure}[t]
    \centering
    \includegraphics[width=.47\textwidth]{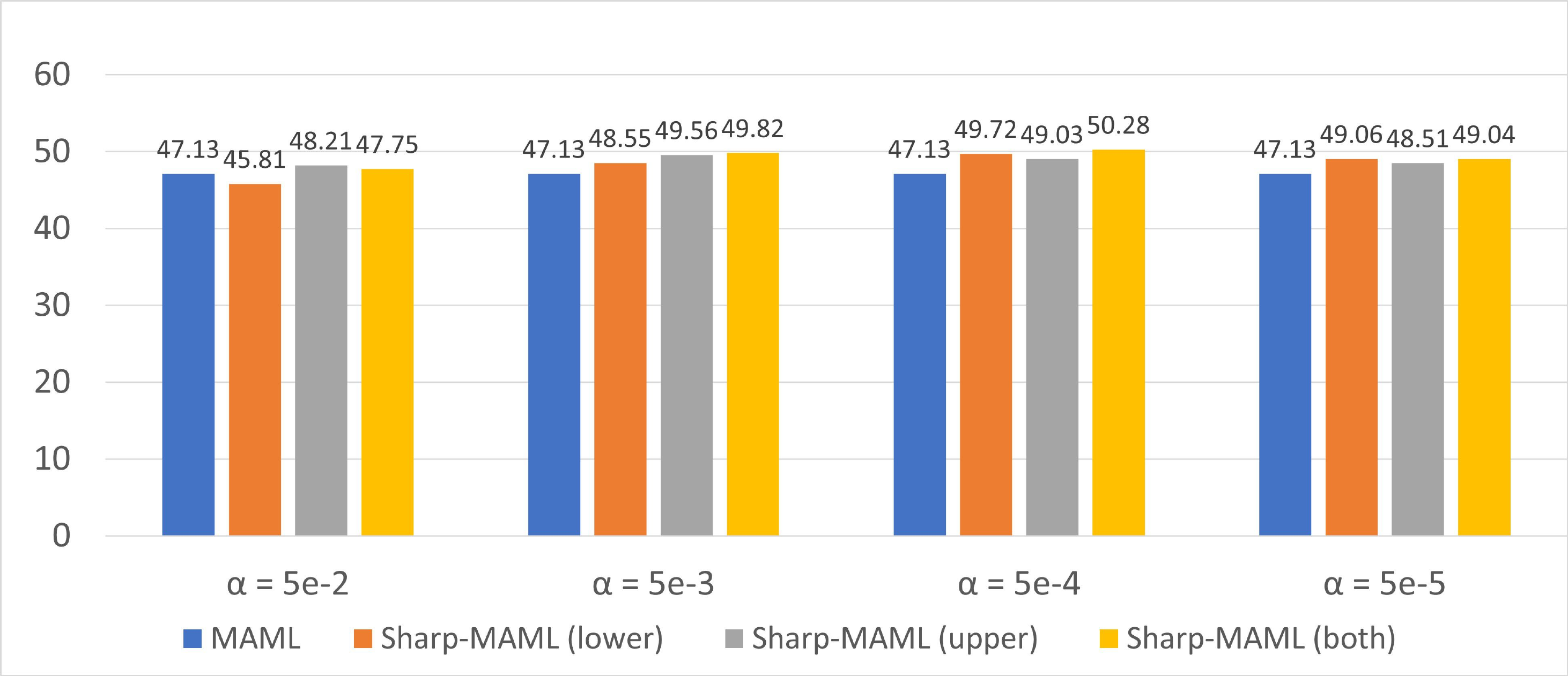}
    \vspace{-0.1cm}
    \caption{ Performance under different values of $\alpha_{\rm low}$, $\alpha_{\rm up}$ on Mini-Imagenet (5-way 1-shot). For Sharp-MAML$_{\rm both}$, we used the same value of $\alpha_{\rm low}$ and $\alpha_{\rm up}$ (i.e., $\alpha_{\rm low}$ = $\alpha_{\rm up}$).}
    \label{figure2}
\vskip -0.05in
\end{figure}

\begin{figure}[t]
    \centering
    \vspace{0.1cm}
    \includegraphics[width=.47\textwidth]{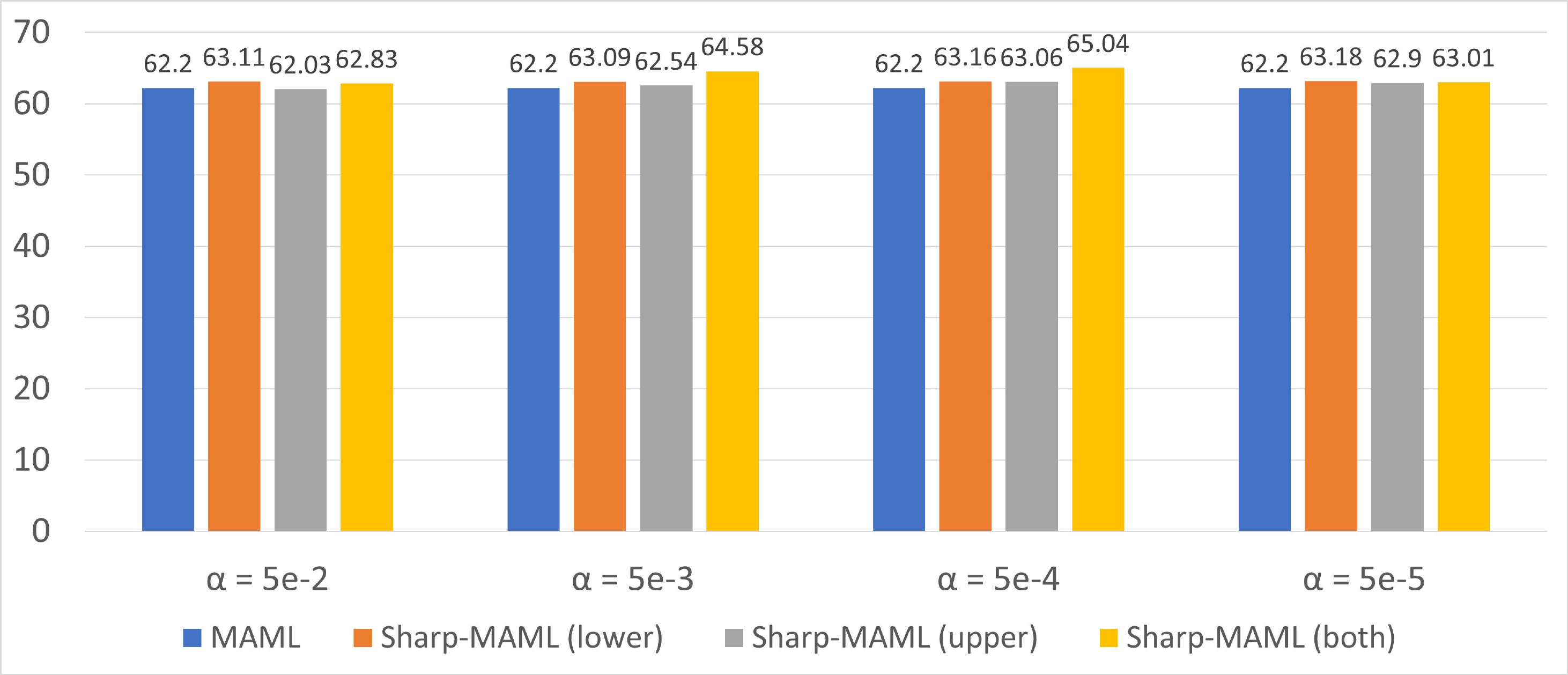}
    \vspace{-0.1cm}
     \caption{ Performance under different values of $\alpha_{\rm low}$, $\alpha_{\rm up}$ on Mini-Imagenet (5-way 5-shot). For Sharp-MAML$_{\rm both}$, we used the same value of $\alpha_{\rm low}$ and $\alpha_{\rm up}$ (i.e., $\alpha_{\rm low}$ = $\alpha_{\rm up}$)}
    \label{figure3}
    \vspace{-0.4cm}
\end{figure}
\begin{figure*}[t]
    \centering
    \includegraphics[width=.9\textwidth]{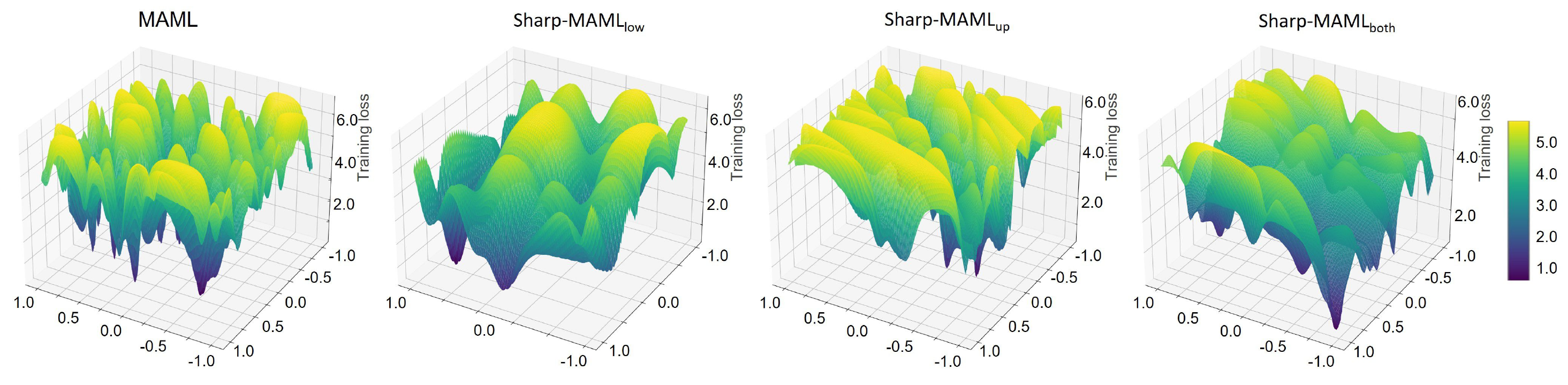}
    \vspace{-0.1cm}
    \caption{One-task cross entropy loss landscapes of MAML and different variants of Sharp-MAML trained on CIFAR-100 dataset (5-way 1-shot setting) using class one. The plots are generated following \citep{Li9}. Details of the architecture are given in Section \ref{sec5.1}.}
    \label{figure4}
    \vspace{-0.1cm}
\end{figure*}

\subsection{Computationally-efficient version of Sharp-MAML} \label{sec5.3}

Next we investigate if the computational overhead of Sharp-MAML$_{\rm low}$ can be further reduced by leveraging the computationally-efficient MAML -- an almost-no-inner-loop (ANIL) method \citep{Raghu} and the computationally-efficient SAM -- ESAM \citep{Du2}. Sharp-MAML-ANIL is the case when we use ANIL with Sharp-MAML$_{\rm low}$; ESharp-MAML is the case when we use ESAM with Sharp-MAML$_{\rm low}$; ESharp-MAML-ANIL is Sharp-MAML$_{\rm low}$ with both ANIL and ESAM.

In ANIL, fine-tuning is only applied to the task-specific head with a frozen representation network  from the meta-model. Motivated by \citep{Raghu}, we ask if incorporating Sharp-MAML$_{\rm low}$  in ANIL can ameliorate the computational overhead while preserving the performance gains obtained using the model trained on Sharp-MAML$_{\rm low}$. ANIL decomposes the meta-model $\theta$ into two parts: the representation encoding network denoted by $\theta_r$ and the classification head denoted by $\theta_c$ i.e., $\theta \triangleq [\theta_r, \theta_c]$. Different from \eqref{MAMLb}, ANIL then only fine-tunes  $\theta_c$ over a specific task $m$, given by:
\begin{equation}
     \,\,\, \theta_{m}^{\prime}(\theta) = \arg \min_{\theta_{m};\, \theta_{r, m} = \theta_{r}} \mathcal{L} (\theta_{c, m}, \theta_{r,m}; \mathcal{D}_m).
\end{equation}
In other words, the initialized representation $\theta_r$, which comprises most of the network, is unchanged during fine-tuning.

ESAM leverages two training strategies, Stochastic Weight Perturbation (SWP) and Sharpness-Sensitive Data Selection (SDS). SWP saves computation by stochastically selecting set of weights in each iteration, and SDS judiciously selects a subset of data that is sensitive to sharpness. 
To be specific, SWP uses a gradient mask ${\bf v} = (v_{1},...,v_{M})$ where $v_i \myeq$ Bern$(\xi)$, $\forall i$.
In SDS, instead of computing ${\cal L}_{{\cal N}}(\theta + {\epsilon(\theta}))$ over all samples, ${\cal N}$, a $\emph{subset}$ of samples, ${\cal N}^{+}$, is selected, whose loss values increase the most with ${\epsilon}(\theta)$; that is,
\begin{align*}
&{\cal N}^{+} \triangleq \{(x_i, y_i): l(\theta+{\epsilon};x_i, y_i) - l(\theta;x_i, y_i) > \tau\}\\
&{\cal N}^{-} \triangleq \{(x_i, y_i): l(\theta+{\epsilon};x_i, y_i) - l(\theta;x_i, y_i) < \tau\}
\end{align*}
where the threshold $\tau$ controls the size of ${\cal N}^{+}$. Furthermore, $\mu=|{\cal N}^{+}|/|{\cal N}|$ is ratio of number of selected samples with respect to the batch size and determines the exact value of $\tau$. 
In practice, $\mu$ is selected to maximize efficiency while preserving generalization performance.

In Table \ref{table5}, we report our results on three computationally efficient versions of Sharp-MAML. 
We find that Sharp-MAML-ANIL is comparable in performance to Sharp-MAML while requiring almost $86 \%$ less computation. ESharp-MAML also reduces the computation, but has slight performance loss. We suspect that this is due to the nested structure of the meta-learning problem that adversely affects the two training strategies used in ESAM. We further investigate the effect of both ANIL and ESAM on Sharp-MAML and observe significant reduction in computation ($116 \%$ faster) with slight degradation in performance as compared to Sharp-MAML. When compared to MAML, ESharp-MAML-ANIL performs considerably better (+1.90$\%$ gain in accuracy) while requiring only $20 \%$ more computation.

\section{Conclusions}
In this paper, we study sharpness-aware minimization (SAM) in the context of model-agnostic meta-learning (MAML) through the lens of bilevel optimization. We name our new MAML method Sharp-MAML.  Through a systematic empirical and theoretical study, we find that adding SAM into any/both fine-tuning or/and meta-update stages improves the generalization performance. We further find that incorporating SAM in the fine-tuning stage alone is the best trade-off between performance and computation. To further save computation overhead, we leverage the techniques such as efficient SAM and almost no inner loop to speed up Sharp-MAML, without sacrificing generalization.

\section*{Acknowledgements}
 The work was partially supported by NSF MoDL-SCALE Grant 2134168 and the Rensselaer-IBM AI Research Collaboration (\url{http://airc.rpi.edu}), part of the IBM AI Horizons Network (\url{http://ibm.biz/AIHorizons}).



\nocite{langley00}

\bibliography{example_paper,pac_bayes}
\bibliographystyle{icml2022}




\clearpage
\appendix
\onecolumn
\icmltitle{Supplementary Material for\\
``\FullTitle"}

\section{Omitted Pseudo-code in The Main Manuscript} \label{appendix::A}
In this section, we present the omitted pseudo-code of MAML and two Sharp-MAML algorithms. 

\subsection{MAML algorithm}
The pseudo-code of plain-vanilla MAML is summarized in Algorithm \ref{alg1}.

\begin{algorithm}[htb]
    \normalsize
	\caption{MAML for few-shot supervised learning}  
	\label{alg1}
	\begin{algorithmic}[1]
	\State {\bf Require:} $p( \mathcal{T})$: distribution over tasks 
	\State {\bf Require:} $\beta_{\rm low}, \beta_{\rm up}$: step sizes
    	\For{$t=1,\cdots, T$}
            \State Sample batch of tasks $\mathcal{T}_{m} \sim p(\mathcal{T})$
            \For{{\bf all} $\mathcal{T}_{m}$}
            \State Sample $K$ examples from $\mathcal{D}_{m} = \{{x}_i, y_{i}\}$
            \State Evaluate $\widetilde\nabla \mathcal{L} ({\theta^t}; \mathcal{D}_{m} )$ 
            \State Compute fine-tuned parameter $\theta_{m}^{\prime}(\theta^t)$ via \eqref{eq2}
            \State Sample datapoints from $\mathcal{D}_m^{\prime} =  \{{x}_{i}, y_{i}\} $  for meta-update
            \EndFor
        \State Update the meta-model $\theta$ by $\theta^{t+1}=\theta^t - \beta_{\rm up} \widetilde\nabla_{\theta^t} \sum_{m=1}^{M} \mathcal{L}(\theta^\prime_m(\theta^t);\mathcal{D}_m^{\prime})$
        \EndFor
        
    \end{algorithmic} 
\end{algorithm}

\begin{algorithm}[htb]
    \normalsize
	\caption{Sharp-MAML$_{\rm up}$}
	\label{alg3}
	\begin{algorithmic}[1]
	\State {\bf Require:} $p( \mathcal{T})$: distribution over tasks 
	\State {\bf Require:} $\beta_{\rm low}, \beta_{\rm up}$:   step sizes
	\State {\bf Require:} $\alpha_{\rm low}>0, \alpha_{\rm up} >0$: perturbation radii
    	\For{$t=1,\cdots, T$}
            \State Sample batch of tasks $\mathcal{T}_{m} \sim p(\mathcal{T})$
            \For{{\bf all} $\mathcal{T}_{m}$}
            \State Sample $K$ examples from  $\mathcal{D}_{m}$ 
            \State Evaluate $\widetilde\nabla \mathcal{L}({\theta^t}; \mathcal{D}_{m})$ 
            \State Compute fine-tuned parameter $\tilde\theta^1_m(\theta^t)=\theta^t-\beta_{\rm low}\widetilde\nabla\mathcal{L}(\theta^t;\mathcal{D}_m)$
            \State Sample data from $\mathcal{D}_m^{\prime}$ for meta-update
            \EndFor
        \State \textcolor{blue} {Compute $ \widetilde\nabla_{\theta^t} \sum_{m=1}^M \mathcal{L}({\tilde\theta^1_m(\theta^t)};\mathcal{D}_m^{\prime})$} 
        \State \textcolor{blue} {Compute perturbation ${\epsilon}(\theta^t)$ via} \eqref{eq_upe}
        \State \textcolor{blue} {Update $\theta^{t+1}$ via \eqref{eq_upmeta}}
        \EndFor
    \end{algorithmic} 
\end{algorithm}

\begin{algorithm}[htb]
    \normalsize
	\caption{Sharp-MAML$_{\rm low}$} 
	\label{alg2}
	\begin{algorithmic}[1]
	\State {\bf Require:} $p( \mathcal{T})$: distribution over tasks 
	\State {\bf Require:} $\beta_{\rm low}, \beta_{\rm up}$:   step sizes
	\State {\bf Require:} $\alpha_{\rm low}>0, \alpha_{\rm up} >0$: perturbation radii
    	\For{$t=1,\cdots, T$}
            \State Sample batch of tasks $\mathcal{T}_{m} \sim p(\mathcal{T})$
            \For{{\bf all} $\mathcal{T}_{m}$}
            \State Sample $K$ examples  $\mathcal{D}_{m}$ from  $\mathcal{T}_m$
            \State Evaluate $\widetilde\nabla \mathcal{L}({\theta^t}; \mathcal{D}_{m})$ 
            \State \textcolor{red}{Compute perturbation ${\epsilon_m}(\theta^t)$ via} \eqref{eq7}
            \State \textcolor{red} {Compute fine-tuned parameter $\tilde\theta_{m}^1(\theta^t)$ via} \eqref{eq_lowtheta}
            \State Sample data $\mathcal{D}_m^{\prime}$ for meta-update
            \EndFor
        \State \textcolor{blue}{Update the meta-model $\theta^{t+1}$ via \eqref{eq_lowmeta}}
        \EndFor
    \end{algorithmic} 
\end{algorithm}

\subsection{Sharp-MAML$_{\rm up}$ algorithm}
In this case, $\epsilon_m(\theta^t)=0,\forall m, t$, so we have that $\tilde\theta^1_m(\theta^t)=\theta^t-\beta_{\rm low}\widetilde\nabla\mathcal{L}(\theta^t;\mathcal{D}_m)$ and $\tilde\theta_m^2(\theta^t)=\theta^t+\epsilon(\theta^t)-\beta_{\rm low}\widetilde\nabla\mathcal{L}(\theta^t+\epsilon(\theta^t);\mathcal{D}_m)$. 
Defining $\nabla h = \widetilde\nabla_{\theta^t} \sum_{m=1}^{M} \mathcal{L}(\tilde\theta^1_m(\theta^t);\mathcal{D}_m^{\prime})$, the upper perturbation $\epsilon (\theta^t)$ can be computed by:
\begin{equation}  \label{eq_upe}
  \epsilon (\theta^t) =  \alpha_{\rm up}{\nabla h}/{||\nabla h||_2}.
\end{equation}
Let $\epsilon^t=\epsilon(\theta^t)$, then the final meta update can be written as
\begin{equation}  \label{eq_upmeta}
    \theta^{t+1}= \theta^t - \beta_{\rm up} \widetilde\nabla_{\theta^t} \sum_{m=1}^{M} \mathcal{L}(\theta^t+\epsilon^t-\beta_{\rm low}\widetilde\nabla\mathcal{L}(\theta^t+\epsilon^t;\mathcal{D}_m);\mathcal{D}_m^{\prime}).
\end{equation}
The pseudo-code is summarized in Algorithm \ref{alg3}.

\subsection{Sharp-MAML$_{\rm low}$ algorithm}
In this case, $\epsilon(\theta^t)=0,\forall t$, so we have that $\tilde\theta^1_m(\theta^t)=\tilde\theta^2_m(\theta^t)=\theta^t-\beta_{\rm low}\widetilde\nabla\mathcal{L}(\theta^t+\epsilon_m(\theta^t);\mathcal{D}_m)$. Then the final meta update can be written as
\begin{align}  
    &\tilde\theta^1_m(\theta^t)=\theta^t-\beta_{\rm low}\widetilde\nabla\mathcal{L}(\theta^t+\epsilon_m(\theta^t);\mathcal{D}_m)\label{eq_lowtheta}\\
    &\theta^{t+1}= \theta^t - \beta_{\rm up} \widetilde\nabla_{\theta^t} \sum_{m=1}^{M} \mathcal{L}(\tilde\theta^1_m(\theta^t);\mathcal{D}_m^{\prime})\label{eq_lowmeta}
\end{align}
The pseudo-code is summarized in Algorithm \ref{alg2}.

\section{Proof of Lemma \ref{maml-infer}}\label{pf:lm1}
\begin{proof}
Since $\mathcal{L}(\theta;\mathcal{D}_m)\in\mathcal{C}^3$, the stationary point of $\mathcal{L}(\theta;\mathcal{D}_m)$ satisfies 
\begin{equation}\label{erm-grad}
    \nabla \mathcal{L}(\theta;\mathcal{D}_m)=0
\end{equation}
and the local minimizer of $\mathcal{L}(\theta;\mathcal{D}_m)$ satisfies 
\begin{equation}\label{erm-hess}
   \nabla \mathcal{L}(\theta;\mathcal{D}_m)=0 ~~{\rm and} ~~ \nabla^2 \mathcal{L}(\theta;\mathcal{D}_m)\succeq 0. 
\end{equation}
Next we compute the gradient of $\mathcal{L}(\theta_m^\prime(\theta);\mathcal{D}_m)$ according to the chain rule, that is
\begin{align}
    \nabla \mathcal{L}(\theta_m^\prime(\theta);\mathcal{D}_m)&=(I-\beta_{\rm low}\nabla^2 \mathcal{L}(\theta;\mathcal{D}_m))\nabla \mathcal{L}(\theta-\beta_{\rm low}\nabla \mathcal{L}(\theta;\mathcal{D}_m);\mathcal{D}_m)\label{mamlgrad}
\end{align}
and the Hessian of $\mathcal{L}(\theta_m^\prime(\theta);\mathcal{D}_m)$, that is
\begin{align}\label{mamlhess}
    \nabla^2 \mathcal{L}(\theta_m^\prime(\theta);\mathcal{D}_m)&=\nabla\left((I-\beta_{\rm low}\nabla^2 \mathcal{L}(\theta;\mathcal{D}_m))\nabla \mathcal{L}(\theta-\beta_{\rm low}\nabla \mathcal{L}(\theta;\mathcal{D}_m);\mathcal{D}_m)\right)\nonumber\\
    &=-\beta_{\rm low}\nabla^3\mathcal{L}(\theta;\mathcal{D}_m)\nabla \mathcal{L}(\theta-\beta_{\rm low}\nabla \mathcal{L}(\theta;\mathcal{D}_m);\mathcal{D}_m)\nonumber\\
    &~~~~+(I-\beta_{\rm low}\nabla^2 \mathcal{L}(\theta;\mathcal{D}_m))^2 \nabla^2 \mathcal{L}(\theta-\beta_{\rm low}\nabla \mathcal{L}(\theta;\mathcal{D}_m);\mathcal{D}_m).
\end{align}

Plugging \eqref{erm-grad} to \eqref{mamlgrad}, we get that 
$
    \nabla_\theta \mathcal{L}(\theta_m^\prime(\theta);\mathcal{D}_m)=0
$
which implies $\theta$ is also a stationary point for $\mathcal{L}(\theta_m^\prime(\theta);\mathcal{D}_m)$. 

Moreover, plugging \eqref{erm-hess} to \eqref{mamlhess}, we get that
$
\nabla_\theta \mathcal{L}(\theta_m^\prime(\theta);\mathcal{D}_m)=0
$
 and
\begin{align*}
 \nabla_\theta^2 \mathcal{L}(\theta_m^\prime(\theta);\mathcal{D}_m)=(I-\beta_{\rm low}\nabla^2 \mathcal{L}(\theta;\mathcal{D}_m))^2 \nabla^2 \mathcal{L}(\theta;\mathcal{D}_m)\succeq 0
\end{align*}
which implies $\theta$ is also a local minimizer of $\mathcal{L}(\theta_m^\prime(\theta);\mathcal{D}_m)$. 

If $\theta$ is the stationary point for $\mathcal{L}(\theta;\mathcal{D}_m),\forall m\in\mathcal{M}$, we know that $\theta$ is also the stationary point for $\mathcal{L}(\theta_m^\prime(\theta);\mathcal{D}_m),\forall m\in\mathcal{M}$. Thus, $\theta$ is also the stationary point for $\sum_{m=1}^M\mathcal{L}(\theta_m^\prime(\theta);\mathcal{D}_m)$. Likewise, the statement is also true for local minimizer. 
\end{proof}

\section{Convergence Analysis}
\subsection{Convergence analysis of MAML \cite{Finn}}\label{sec:con-MAML}
We provide theoretical analysis for MAML \cite{Finn}. First, we state the exact form of MAML as follows. 
\begin{subequations}  \label{1step-MAML}
\begin{align} 
      \min_{\theta} &~\frac{1}{M} \sum_{m=1}^{M} \mathcal{L} (\theta_{m}^{\prime}(\theta); \mathcal{D}_m^\prime)  \\
  {\rm s. t.} \,\,\, & \theta_{m}^{\prime}(\theta) = \theta - \beta_{\rm low} \nabla \mathcal{L} (\theta; \mathcal{D}_m),~\forall m\in\mathcal{M}. 
\end{align}
 \end{subequations}
The problem \eqref{1step-MAML} can be reformulated as 
\begin{subequations}  \label{1step-MAML-bi}
\begin{align} 
        \min_{\theta} &~~~ F(\theta)=\frac{1}{M} \sum_{m=1}^{M} \mathcal{L} (\theta_{m}^{\prime}(\theta); \mathcal{D}_m^\prime)   \\
\label{bi-LL}
    {\rm s. t.} \,\,\,& \theta_{m}^{\prime}(\theta) = \arg \min_{\theta_m}\left\{ \nabla\mathcal{L}  (\theta; \mathcal{D}_m)^\top(\theta_m-\theta)+ \frac{1}{2\beta_{\rm low}}\|\theta_m-\theta\|^2\right\}.
\end{align}
 \end{subequations}
 
Next, to show the connection between MAML formulation and ALSET \cite{Chen}, we can concatenate $\theta_m$ as a new vector $\phi=[\theta_1^\top,\cdots, \theta_M^\top]^\top$ and define
\begin{align*}
\allowdisplaybreaks
    F(\theta)=f(\phi^\prime(\theta)),~~~f(\phi)=\frac{1}{M}\sum_{m=1}^M\mathcal{L}(\theta_m;\mathcal{D}_m^\prime),~~~ g(\theta,\phi)=\frac{1}{M}\sum_{m=1}^M\left\{\nabla\mathcal{L}(\theta;\mathcal{D}_m)^\top(\theta_m-\theta)+\frac{\|\theta_m-\theta\|^2}{2\beta_{\rm low}}\right\}
\end{align*}
where $\phi^\prime(\theta)=\arg\min_\phi g(\theta,\phi)$. Then the Jacobian and Hessian of $f$ and $g$ can be computed by
\begin{align*}
    \nabla_{\phi}f(\phi)=\left(                 
  \begin{array}{l}   
    \nabla\mathcal{L}(\theta_1;\mathcal{D}_1^\prime) \\  
    ~~~~~~~~~~\vdots \\
    \nabla\mathcal{L}(\theta_M;\mathcal{D}_M^\prime)  
  \end{array}\right), ~~~
  \nabla_{\phi\theta}g\left(\theta,\phi\right)=\left(                 
  \begin{array}{l}   
    \nabla^2\mathcal{L}(\theta_1;\mathcal{D}_1)-\beta_{\rm low}^{-1}I \\  
    ~~~~~~~~~~\vdots \\
    \nabla^2\mathcal{L}(\theta_M;\mathcal{D}_M)-\beta_{\rm low}^{-1}I  
  \end{array}
\right),~~~
\nabla_{\phi\phi}g(\theta,\phi)=\beta^{-1}_{\rm low}I
\end{align*}
where $I$ denotes the identity matrix. According to the expression of $\nabla F(\theta)$ in ALSET \cite{Chen}, we can verify that MAML's gradient has the following form
\begin{align}
    \nabla F(\theta)&=-\nabla_{\theta\phi}g(\theta,\phi)\nabla_{\phi\phi}^{-1}g(\theta,\phi)\nabla_\phi f(\phi)\nonumber\\
    &=\frac{1}{M}\sum_{m=1}^M (I-\beta_{\rm low}\nabla^2\mathcal{L}(\theta;\mathcal{D}_m))\nabla\mathcal{L}(\theta_m;\mathcal{D}_m^\prime).
\end{align}
Moreover, since $g(\theta,\phi)$ is a quadratic function with respect to $\phi$, the strongly convexity and Lipschitz continuity assumptions hold.\footnote{ $\nabla^2 g$ is Lipschitz continuous in Assumption 1 in \cite{Chen} can be reduced to $\nabla_{\phi\phi} g$ and $\nabla_{\phi\theta} g$ is Lipschitz continuous, which can be satisfied under Assumption \ref{li-lip}.} Assumptions about upper-level function also holds under Assumption \ref{li-lip}. 

Then, for notational simplicity, we consider the single-sample case with $K=1$ and define three independent samples for stochastic gradient and Hessian computation as $\xi_m:=(x,y)\sim \mathcal{D}_m,\psi_m:=(x,y)\sim \mathcal{D}_m, \xi^\prime_m:=(x,y)\sim\mathcal{D}_m^\prime$, so the corresponding $K$-batch gradient and Hessian estimators used in MAML algorithms can be written as
\begin{align*}
    &\nabla\mathcal{L}(\theta;\mathcal{D}_m,\xi_m)=\frac{1}{K}\sum_{\xi_m\sim \mathcal{D}_m}\nabla l(\theta,x,y),\\
    &\nabla^2\mathcal{L}(\theta;\mathcal{D}_m,\psi_m)=\frac{1}{K}\sum_{\psi_m \sim \mathcal{D}_m}\nabla^2 l(\theta,x,y),\\ &\nabla\mathcal{L}(\theta;\mathcal{D}^\prime_m,\xi^\prime_m)=\frac{1}{K}\sum_{\xi_m^\prime\sim\mathcal{D}_m^\prime}\nabla l(\theta,x,y).
\end{align*}
Based on these notations, we can write the stochastic update of MAML algorithm \cite{Finn} as
\begin{align*}
    &\theta^{t+1}=\theta^t-\beta_{\rm up}\widetilde\nabla_{\theta^t}\mathcal{L}(\theta^t;\mathcal{D}_m)=\theta^t-\frac{\beta_{\rm up}}{M}\sum_{m=1}^M (I-\beta_{\rm low}\nabla^2\mathcal{L}(\theta^t;\mathcal{D}_m,\psi_m))\nabla\mathcal{L}(\theta_m^{t+1};\mathcal{D}_m^\prime,\xi_m^\prime)\\
    &\theta_m^{t+1}=\theta^t-\beta_{\rm low}\nabla\mathcal{L}(\theta^t;\mathcal{D}_m,\xi_m).
\end{align*}

Then MAML algorithm can be seen as a special case of ALSET, so we have the following Lemma. 

\begin{lemma}
    Under Assumption \ref{li-lip}--\ref{variance}, and choosing stepsizes $\beta_{\rm low},\beta_{\rm up}={\cal O}(\frac{1}{\sqrt{T}})$ with some proper constants, we can get that the iterates $\{\theta^t\}$ generated by MAML \cite{Finn} satisfy
    \begin{align*}
        \frac{1}{T}\sum_{t=1}^T\mathbb{E}\left[\|\nabla F(\theta^t)\|^2\right]={\cal O}\left(\frac{1}{\sqrt{T}}\right). 
    \end{align*}
\end{lemma}

\subsection{Convergence analysis of a generic biased MAML}\label{sec:con-biasMAML}

Since Sharp-MAML can be treated as biased update version of MAML \cite{Finn}, we first analyze the general biased MAML algorithm. Suppose that biased MAML update with
\begin{align*}
    &\theta^{t+1}=\theta^t-\frac{\beta_{\rm up}}{M}\sum_{m=1}^M (I-\beta_{\rm low}\hat\nabla^2\mathcal{L}(\theta^t;\mathcal{D}_m,\psi_m))\nabla\mathcal{L}(\hat\theta_m^{t+1};\mathcal{D}_m^\prime,\xi_m^\prime)\\
    &\hat\theta_m^{t+1}=\theta^t-\beta_{\rm low}\hat\nabla\mathcal{L}(\theta^t;\mathcal{D}_m,\xi_m)
\end{align*}
where $\hat\nabla^2\mathcal{L}(\theta^t;\mathcal{D}_m,\psi_m), \hat\nabla\mathcal{L}(\theta^t;\mathcal{D}_m,\xi_m)$ are biased estimator of $\nabla^2\mathcal{L}(\theta^t;\mathcal{D}_m), \nabla\mathcal{L}(\theta^t;\mathcal{D}_m)$, respectively. With the notation that 
$$\hat\nabla\mathcal{L}(\theta;\mathcal{D}_m)\triangleq\mathbb{E}_{\xi_m}\left[\hat\nabla\mathcal{L}(\theta;\mathcal{D}_m,\xi_m)\right],\hat\nabla^2\mathcal{L}(\theta;\mathcal{D}_m)\triangleq\mathbb{E}_{\psi_m}\left[\hat\nabla^2\mathcal{L}(\theta;\mathcal{D}_m,\psi_m)\right],$$
we make the following assumptions. 

\begin{assumption}[Stochastic derivatives]\label{varaince-bias-maml}
Assume that $\hat\nabla\mathcal{L}(\theta;\mathcal{D}_m,\xi_m),\hat\nabla^2\mathcal{L}(\theta;\mathcal{D}_m,\psi_m)$ are unbiased estimator of $\hat\nabla\mathcal{L}(\theta;\mathcal{D}_m),\hat\nabla^2\mathcal{L}(\theta;\mathcal{D}_m)$ respectively and their variances are bounded by $\sigma_b^2$.  
\end{assumption}

\begin{assumption}\label{bias-maml}
Assume that $\|\hat\nabla^2 \mathcal{L}(\theta^t;\mathcal{D}_m)-\nabla^2 \mathcal{L}(\theta^t;\mathcal{D}_m)\|\leq \gamma_h$, $\|\hat\nabla \mathcal{L}(\theta^t;\mathcal{D}_m)-\nabla \mathcal{L}(\theta^t;\mathcal{D}_m)\|\leq \gamma_g,\forall m\in\mathcal{M}$. 
\end{assumption}

Throughout the proof, we use
\begin{align*}
    &\mathcal{F}_{t}=\sigma\left\{\hat\theta_1^{0},\cdots,\hat\theta_M^0, \theta^{0}, \ldots,\theta^t, \hat\theta_1^{t+1},\ldots,\hat\theta_M^{t+1}\right\},~~ \mathcal{F}_{t}^\prime=\sigma\left\{\hat\theta_1^{0},\cdots,\hat\theta_M^0, \theta^{0}, \ldots,\theta^t\right\}
\end{align*}
where $\sigma\{\cdot\}$ denotes the $\sigma$-algebra generated by random variables. We also denote
\begin{align*}
h^t&\triangleq\frac{1}{M}\sum_{m=1}^M (I-\beta_{\rm low}\hat\nabla^2\mathcal{L}(\theta^t;\mathcal{D}_m,\psi_m))\nabla\mathcal{L}(\hat\theta_m^{t+1};\mathcal{D}_m^\prime,\xi_m^\prime)\\
\bar h^t&\triangleq\mathbb{E}\left[\frac{1}{M}\sum_{m=1}^M (I-\beta_{\rm low}\hat\nabla^2\mathcal{L}(\theta^t;\mathcal{D}_m,\psi_m))\nabla\mathcal{L}(\hat\theta_m^{t+1};\mathcal{D}_m^\prime,\xi_m^\prime)\Big|\mathcal{F}_t\right]\\
&=\frac{1}{M}\sum_{m=1}^M (I-\beta_{\rm low}\hat\nabla^2\mathcal{L}(\theta^t;\mathcal{D}_m))\nabla\mathcal{L}(\hat\theta_m^{t+1};\mathcal{D}_m^\prime).
\end{align*}

\begin{lemma}\label{bs:ht}
Under Assumption \ref{li-lip}--\ref{bias-maml}, we have that
    \begin{align}\label{eq:bsht}
        \mathbb{E}\left[\|\nabla F(\theta^t)-\bar h^t\|^2\right]\leq 4\ell_1^2\beta_{\rm low}^2(\gamma_g^2+\sigma_b^2)+4\beta_{\rm low}^2\left(\ell_0^2\gamma_h^2+4(\gamma_h^2+\ell_1^2)\ell_1^2\beta_{\rm low}^2(\gamma_g^2+\sigma_b^2)\right).
    \end{align}
\end{lemma}
\begin{proof}
Since $\mathbb{E}\left[\hat\theta_m^{t+1}|\mathcal{F}_t^\prime\right]=\theta^t-\beta_{\rm low}\hat\nabla\mathcal{L}(\theta^t;\mathcal{D}_m)$, then from Assumption 4, we  have $\left\|\theta_m^\prime(\theta^t)-\mathbb{E}\left[\hat\theta_m^{t+1}|\mathcal{F}_t^\prime\right]\right\|\leq \beta_{\rm low}\gamma_g$. Taking expectation with respect to $\mathcal{F}_t^\prime$, we get $$\mathbb{E}\left[\|\theta_m^\prime(\theta^t)-\hat\theta_m^{t+1}\|^2\right]\leq2\mathbb{E}\left[\|\theta_m^\prime(\theta^t)-\mathbb{E}\left[\hat\theta_m^{t+1}|\mathcal{F}_t^\prime\right]\|^2\right]+2\mathbb{E}\left[\|\mathbb{E}\left[\hat\theta_m^{t+1}|\mathcal{F}_t^\prime\right]-\hat\theta_m^{t+1}\|^2\right]\leq 2\beta_{\rm low}^2\gamma_g^2+2\beta_{\rm low}^2\sigma_b^2. $$
Then using Lipschitz continuity of $\nabla\mathcal{L}(\theta;\mathcal{D}_m^\prime)$, we obtain
\begin{align}\label{eq--28}
    \mathbb{E}\left\|\nabla\mathcal{L}(\theta_m^\prime(\theta^t);\mathcal{D}_m^\prime)-\nabla\mathcal{L}(\hat\theta_m^{t+1};\mathcal{D}_m^\prime)\right\|^2\leq2\ell_1^2\beta_{\rm low}^2(\gamma_g^2+\sigma_b^2). 
\end{align}
On the other hand, by observing that $$\left\|\hat\nabla^2\mathcal{L}(\theta^t;\mathcal{D}_m)\right\|^2\leq 2\left\|\nabla^2\mathcal{L}(\theta^t;\mathcal{D}_m)\right\|^2+2\left\|\nabla^2\mathcal{L}(\theta^t;\mathcal{D}_m)-\hat\nabla^2\mathcal{L}(\theta^t;\mathcal{D}_m)\right\|^2\leq2(\gamma_h^2+\ell_1^2),$$
we get 
\begin{align}
    &~~~~~~\mathbb{E}\left\|\nabla^2\mathcal{L}(\theta^t;\mathcal{D}_m)\nabla\mathcal{L}(\theta_m^\prime(\theta^t);\mathcal{D}_m^\prime)-\hat\nabla^2\mathcal{L}(\theta^t;\mathcal{D}_m)\nabla\mathcal{L}(\hat\theta_m^{t+1};\mathcal{D}_m^\prime)\right\|^2\nonumber\\
    &\leq2\mathbb{E}\left\|\nabla^2\mathcal{L}(\theta^t;\mathcal{D}_m)\!-\!\hat\nabla^2\mathcal{L}(\theta^t;\mathcal{D}_m)\right\|^2\left\|\nabla\mathcal{L}(\theta_m^\prime(\theta^t);\mathcal{D}_m^\prime)\right\|^2\!+\!2\mathbb{E}\left\|\hat\nabla^2\mathcal{L}(\theta^t;\mathcal{D}_m)\right\|^2\left\|\nabla\mathcal{L}(\theta_m^\prime(\theta^t);\mathcal{D}_m^\prime)\!-\!\nabla\mathcal{L}(\hat\theta_m^{t+1};\mathcal{D}_m^\prime)\right\|^2\nonumber\\
    &\leq 2\ell_0^2\gamma_h^2+8(\gamma_h^2+\ell_1^2)\ell_1^2\beta_{\rm low}^2(\gamma_g^2+\sigma_b^2).\label{eq--29}
\end{align}
Thus, using \eqref{eq--28} and \eqref{eq--29}, we get
\begin{align*}
    &~~~~~~\mathbb{E}\left[\|\nabla F(\theta^t)-\bar h^t\|^2\right]\\
    &=\mathbb{E}\left\|\frac{1}{M}\sum_{m=1}^M\left[ (I-\beta_{\rm low}\nabla^2\mathcal{L}(\theta^t;\mathcal{D}_m))\nabla\mathcal{L}(\theta_m^\prime(\theta^t);\mathcal{D}_m^\prime)-(I-\beta_{\rm low}\hat\nabla^2\mathcal{L}(\theta^t;\mathcal{D}_m))\nabla\mathcal{L}(\hat\theta_m^{t+1};\mathcal{D}_m^\prime)\right]\right\|^2\\
    &\leq\frac{2}{M}\sum_{m=1}^M\mathbb{E}\left\|\nabla\mathcal{L}(\theta_m^\prime(\theta^t);\mathcal{D}_m^\prime)-\nabla\mathcal{L}(\hat\theta_m^{t+1};\mathcal{D}_m^\prime)\right\|^2\\
    &~~~+\frac{2\beta_{\rm low}^2}{M}\sum_{m=1}^M\mathbb{E}\left\|\nabla^2\mathcal{L}(\theta^t;\mathcal{D}_m)\nabla\mathcal{L}(\theta_m^\prime(\theta^t);\mathcal{D}_m^\prime)-\hat\nabla^2\mathcal{L}(\theta^t;\mathcal{D}_m)\nabla\mathcal{L}(\hat\theta_m^{t+1};\mathcal{D}_m^\prime)\right\|^2\\
    &\leq 4\ell_1^2\beta_{\rm low}^2(\gamma_g^2+\sigma_b^2)+4\beta_{\rm low}^2\left(\ell_0^2\gamma_h^2+4(\gamma_h^2+\ell_1^2)\ell_1^2\beta_{\rm low}^2(\gamma_g^2+\sigma_b^2)\right)
\end{align*}
from which the proof is complete. 
\end{proof}

\begin{lemma}\label{biasmaml-con}
    Under Assumption \ref{li-lip}--\ref{bias-maml}, and choosing stepsizes $\beta_{\rm low},\beta_{\rm up}={\cal O}(\frac{1}{\sqrt{T}})$, and  $\gamma_g,\gamma_h={\cal O}(1)$ with some proper constants, we can get that the iterates $\{\theta^t\}$ generated by biased MAML satisfy
    \begin{align*}
        \frac{1}{T}\sum_{t=1}^T\mathbb{E}\left[\|\nabla F(\theta^t)\|^2\right]={\cal O}\left(\frac{1}{\sqrt{T}}\right). 
    \end{align*}
\end{lemma}
\begin{proof}\allowdisplaybreaks
First we bound the variance of stochastic biased meta gradient estimator $h^t$ as
\begin{align*}
    \mathbb{E}\left[\|h^t-\bar h^t\|^2\Big|\mathcal{F}_t\right]&\leq \frac{2}{M}\sum_{m=1}^M\mathbb{E}\left[\|(I-\beta_{\rm low}\hat\nabla^2\mathcal{L}(\theta^t;\mathcal{D}_m,\psi_m))\nabla\mathcal{L}(\hat\theta_m^{t+1};\mathcal{D}_m^\prime,\xi_m^\prime)\right.\\
    &~~~~\left.-(I-\beta_{\rm low}\hat\nabla^2\mathcal{L}(\theta^t;\mathcal{D}_m))\nabla\mathcal{L}(\hat\theta_m^{t+1};\mathcal{D}_m^\prime)\|^2\Big|\mathcal{F}_t\right]\\
    &\leq \frac{4}{M}\sum_{m=1}^M\mathbb{E}\left[\|\nabla\mathcal{L}(\hat\theta_m^{t+1};\mathcal{D}_m^\prime,\xi_m^\prime)-\nabla\mathcal{L}(\hat\theta_m^{t+1};\mathcal{D}_m^\prime)\|^2\Big|\mathcal{F}_t\right]\\
    &~~~~+\frac{4\beta_{\rm low}^2}{M}\sum_{m=1}^M\mathbb{E}\left[\|\hat\nabla^2\mathcal{L}(\theta^t;\mathcal{D}_m,\psi_m)\nabla\mathcal{L}(\hat\theta_m^{t+1};\mathcal{D}_m^\prime,\xi_m^\prime)-\hat\nabla^2\mathcal{L}(\theta^t;\mathcal{D}_m)\nabla\mathcal{L}(\hat\theta_m^{t+1};\mathcal{D}_m^\prime)\|^2\Big|\mathcal{F}_t\right]\\
    &\leq 4\ell_1^2\sigma^2+\frac{8\beta_{\rm low}^2}{M}\sum_{m=1}^M\mathbb{E}\left[\|\nabla\mathcal{L}(\hat\theta_m^{t+1};\mathcal{D}_m^\prime,\xi_m^\prime)\|^2\|\hat\nabla^2\mathcal{L}(\theta^t;\mathcal{D}_m,\psi_m)-\hat\nabla^2\mathcal{L}(\theta^t;\mathcal{D}_m)\|^2\Big|\mathcal{F}_t\right]\\
    &~~~~~+\frac{8\beta_{\rm low}^2}{M}\sum_{m=1}^M\mathbb{E}\left[\|\hat\nabla^2\mathcal{L}(\theta^t;\mathcal{D}_m)\|^2\|\nabla\mathcal{L}(\hat\theta_m^{t+1};\mathcal{D}_m^\prime,\xi_m^\prime)-\nabla\mathcal{L}(\hat\theta_m^{t+1};\mathcal{D}_m^\prime)\|^2\Big|\mathcal{F}_t\right]\\
    &\leq 4\ell_1^2\sigma^2+ 8\beta_{\rm low}^2\sigma_b^2\mathbb{E}\left[\|\nabla\mathcal{L}(\hat\theta_m^{t+1};\mathcal{D}_m^\prime,\xi_m^\prime)\|^2\Big|\mathcal{F}_t\right]+16(\gamma_h^2+\ell_1^2)\beta_{\rm low}^2\sigma^2\\
    &\leq 4\ell_1^2\sigma^2+16(\gamma_h^2+\ell_1^2)\beta_{\rm low}^2\sigma^2+ 16\beta_{\rm low}^2\sigma_b^2 (\sigma^2+\ell_0^2)\triangleq\tilde\sigma^2 \numberthis\label{tildesigma}
\end{align*}
where \eqref{tildesigma} comes from
\begin{align*}
    \mathbb{E}\left[\|\nabla\mathcal{L}(\hat\theta_m^{t+1};\mathcal{D}_m^\prime,\xi_m^\prime)\|^2\Big|\mathcal{F}_t\right]&\leq 2\|\nabla\mathcal{L}(\hat\theta_m^{t+1};\mathcal{D}_m^\prime)\|^2+2\mathbb{E}\left[\|\nabla\mathcal{L}(\hat\theta_m^{t+1};\mathcal{D}_m^\prime,\xi_m^\prime)-\nabla\mathcal{L}(\hat\theta_m^{t+1};\mathcal{D}_m^\prime)\|^2\Big|\mathcal{F}_t\right]\leq 2(\ell_0^2+\sigma^2).
\end{align*}

Then according to Lemma 4 in \cite{Chen}, $F$ is smooth with constant $L_F=\mathcal{O}\left(1\right)$. Using the smoothness property with Lemma 2 in \cite{Chen}, it follows that 
\begin{align*}
\mathbb{E}\left[F(\theta^{t+1})|\mathcal{F}_t\right] & \leq F(\theta^{t})+\mathbb{E}\left[\left\langle\nabla  F(\theta^{t}), \theta^{t+1}-\theta^{t}\right\rangle|\mathcal{F}_t\right]+\frac{L_{F}}{2} \mathbb{E}\left[\left\|\theta^{t+1}-\theta^{t}\right\|^{2}|\mathcal{F}_t\right]  \\
&\leq  F(\theta^{t})-\beta_{\rm up}\left\langle\nabla  F(\theta^{t}), \bar h^t\right\rangle+\frac{L_{F} \beta_{\rm up}^{2}}{2}\mathbb{E}\left[\|h^t\|^2|\mathcal{F}_t\right]\\
&\stackrel{(a)}{\leq} F(\theta^{t})-\beta_{\rm up}\left\langle\nabla  F(\theta^{t}), \bar h^t\right\rangle+\frac{L_{F} \beta_{\rm up}^{2}}{2}\left\|\bar h^t\right\|^{2}+\frac{L_{F} \beta_{\rm up}^{2}\tilde\sigma^2}{2}\\
& \stackrel{(b)}{=}  F(\theta^{t})-\frac{\beta_{\rm up}}{2}\left\|\nabla  F(\theta^{t})\right\|^{2}-\left(\frac{\beta_{\rm up}}{2}-\frac{L_{F} \beta_{\rm up}^{2}}{2}\right)\left\|\bar h^t\right\|^{2}+\frac{\beta_{\rm up}}{2}\left\|\nabla  F(\theta^{t})-\bar h^t\right\|^{2}+\frac{L_{F} \beta_{\rm up}^{2}\tilde\sigma^2}{2}
\end{align*}
where (a) uses $\mathbb{E}\left[\|A\|^2|B\right]=(\mathbb{E}\left[A|B\right])^2+\mathbb{E}\left[\|A-\mathbb{E}\left[A|B\right]\|^2|B\right]$ and \eqref{tildesigma}, and (b) uses $2 a^{\top} b=\|a\|^{2}+\|b\|^{2}-\|a-b\|^{2}$. Then using the result in Lemma \ref{bs:ht} and choosing $\beta_{\rm up}\leq\frac{1}{L_F}$, telescoping and rearranging it, we obtain that
\begin{align}\label{conv-maml}
    \frac{1}{T}\sum_{t=1}^T \mathbb{E}\left[\left\|\nabla F(\theta^{t})\right\|^{2}\right]\leq \frac{2F(\theta^1)}{\beta_{\rm up}T}+L_{F} \beta_{\rm up}\tilde\sigma^2+4\beta_{\rm low}^2\left(\ell_1^2(\gamma_g^2+\sigma_b^2)+\ell_0^2\gamma_h^2+4(\gamma_h^2+\ell_1^2)\ell_1^2\beta_{\rm low}^2(\gamma_g^2+\sigma_b^2)\right).
\end{align}
Choosing $\beta_{\rm up},\beta_{\rm low}={\cal O}(\frac{1}{\sqrt{T}})$ and $\gamma_g,\gamma_h={\cal O}(1)$, we can get the ${\cal O}(\frac{1}{\sqrt{T}})$ convergence results of biased MAML. 
\end{proof}

\subsection{Convergence analysis of Sharp-MAML}\label{smaml-convergence}
Thanks to the discussion in Section \ref{sec:con-biasMAML}, the convergence of Sharp-MAML is straightforward. Here we only prove for Sharp-MAML$_{\rm both}$ since same results for the other two variants can be derived by setting $\alpha_{\rm up}=0$ or $\alpha_{\rm low}=0$ accordingly. 

\begin{lemma}
    Under Assumption \ref{li-lip}--\ref{variance}, and choosing stepsizes $\beta_{\rm low},\beta_{\rm up}={\cal O}(\frac{1}{\sqrt{T}})$ and perturbation radii $\alpha_{\rm up}={\cal O}(\frac{1}{\sqrt{T}})$, $\alpha_{\rm low}={\cal O}(1)$,  with some proper constants, we can get that the iterates $\{\theta^t\}$ generated by Sharp-MAML$_{\rm both}$ satisfy
    \begin{align*}
        \frac{1}{T}\sum_{t=1}^T\mathbb{E}\left[\|\nabla F(\theta^t)\|^2\right]={\cal O}\left(\frac{1}{\sqrt{T}}\right). 
    \end{align*}
\end{lemma}
\begin{proof}\allowdisplaybreaks
Recalling the update of Sharp-MAML$_{\rm both}$, we rewrite it as follows. 
\begin{align*}
    &\theta^{t+1}=\theta^t-\frac{\beta_{\rm up}}{M}\sum_{m=1}^M (I-\beta_{\rm low}\nabla^2\mathcal{L}(\theta^t+\epsilon(\theta^t)+\epsilon_m(\theta^t);\mathcal{D}_m,\psi_m))\nabla\mathcal{L}(\hat\theta_m^{t+1};\mathcal{D}_m^\prime,\xi_m^\prime);\\
    &\hat\theta_m^{t+1}=\theta^t+\epsilon(\theta^t)-\beta_{\rm low}\nabla\mathcal{L}(\theta^t+\epsilon(\theta^t)+\epsilon_m(\theta^t);\mathcal{D}_m,\xi_m)\\
    &~~~~~~~~=\theta^t-\beta_{\rm low}\left(\nabla\mathcal{L}(\theta^t+\epsilon(\theta^t)+\epsilon_m(\theta^t);\mathcal{D}_m,\xi_m)-\frac{\epsilon(\theta^t)}{\beta_{\rm low}}\right).
\end{align*}
Since $\nabla\mathcal{L}(\theta;\mathcal{D}_m),\nabla^2\mathcal{L}(\theta;\mathcal{D}_m)$ are Lipschitz continuous with $\ell_1,\ell_2$ according to Assumption \ref{li-lip}, then we have that
\begin{align*}
    &\|\nabla\mathcal{L}(\theta^t+\epsilon(\theta^t)+\epsilon_m(\theta^t);\mathcal{D}_m)-\frac{\epsilon(\theta^t)}{\beta_{\rm low}}-\nabla\mathcal{L}(\theta^t;\mathcal{D}_m)\|\leq \ell_1(\alpha_{\rm up}+\alpha_{\rm low})+\frac{\alpha_{\rm up}}{\beta_{\rm low}}\\
    &\|\nabla^2\mathcal{L}(\theta^t+\epsilon(\theta^t)+\epsilon_m(\theta^t);\mathcal{D}_m)-\nabla^2\mathcal{L}(\theta^t;\mathcal{D}_m)\|\leq \ell_2(\alpha_{\rm up}+\alpha_{\rm low}),
\end{align*}
which satisfies the condition in Lemma \ref{biasmaml-con} if $\alpha_{\rm up}={\cal O}(\frac{1}{\sqrt{T}}),\alpha_{\rm low}={\cal O}(1)$. 
Thus, we arrive at the conclusion. 
\end{proof}

\section{Generalization Analysis}
We build on the recently developed PAC-Bayes bound for meta learning \cite{farid2021generalization}, as restated below.

\begin{lemma}
    \label{lm:pac_bayes_ml}
    Assume the loss function $\mathcal{L}(\cdot)$ is bounded: $0 \leq \mathcal{L}(h, \mathcal{D}) \leq 1$ for any $h$ in the hypothesis space, and any $\mathcal{D}$ in the sample space. 
    For hypotheses $h_{A(\theta, \mathcal{D})}$ learned with $\gamma_{\mathrm{A}}$ uniformly stable algorithm $A$, data-independent prior $P_{\theta, 0}$ over initializations $\theta$, 
    and $\delta \in(0,1)$, with probability at least $1-\delta$ over a sampling of the meta-training dataset $\mathcal{D} \sim \mathcal{P}$, and $|\mathcal{D}| = n$, which include meta-training data from all tasks, the following holds simultaneously for all distributions $P_{\theta}$ over $\theta$:
    \begin{align*}
        \underset{\mathcal{D} \sim \mathcal{P}}{\mathbb{E}}
        ~~\underset{\theta \sim P_{\theta}}{\mathbb{E}}
        \mathcal{L}(h_{A (\theta, \mathcal{D})}, \mathcal{D}) \leq 
        \frac{1}{n} \sum_{i=1}^{n} \underset{\theta \sim P_{\theta}}{\mathbb{E}} 
        \mathcal{L}(h_{A (\theta, \mathcal{D})}, (x_i,y_i))
        +\sqrt{\frac{D_{\mathrm{KL}}(P_{\theta} \| P_{\theta, 0})+\ln \frac{2 \sqrt{n}}{\delta}}{2 n}}
        +\gamma_{\mathrm{A}}.
    \end{align*}
\end{lemma}

We can obtain a generalization bound below.
\begin{theorem}
    Assume loss function $\mathcal{L}(\cdot)$ is bounded: $0 \leq \mathcal{L}(\theta_m; \mathcal{D}) \leq 1$ for any $\theta_m$, and any $\mathcal{D}$, and
    $\mathcal{L}(\theta_m (\hat{\theta});\mathcal{P} ) \leq \mathbb{E}_{\epsilon \sim \mathcal{N}(\mathbf{0}, \alpha^2\mathbf{I})}[\mathcal{L}(\theta_m (\hat{\theta}+{\epsilon}) ;\mathcal{P})]$ 
    at the stationary point of the Sharp-MAML$_{\rm up}$ denoted by  $\hat{\theta}$.
    For parameter $\theta_m(\hat{\theta}; \mathcal{D})$ learned with $\gamma_{\mathrm{A}}$ uniformly stable algorithm $A$ 
    from $\hat{\theta} \in \mathbb{R}^k$, 
with probability $1-\delta$ over the choice of the training set $\mathcal{D} \sim \mathcal{P}$, with $|\mathcal{D}| = n$, it holds that
\begin{align}\label{eq:thm_pac_bayes_ml}
&\mathcal{L} (\theta_m(\hat{\theta});{\mathcal{P}}) \leq 
\max _{\|\epsilon\|_{2} \leq \alpha} \mathcal{L}(\theta_m(\hat{\theta}+{\epsilon}); \mathcal{D})
+ \gamma_{\rm A}+
\sqrt{\frac{ k \ln \Big(1+\frac{\|\hat{\theta}\|_{2}^{2}}{ \alpha^{2}} \Big(1+\sqrt{\frac{\ln n}{k}}\Big)^{2} \Big) + 2\ln \frac{1}{\delta} + 5 \ln n + \mathcal{O}(1)}{4n}}.   \nonumber
\end{align}
\end{theorem}

\begin{proof}
    Since Lemma~\ref{lm:pac_bayes_ml} holds for any prior $P_{\theta,0}$ and posterior $P_{\theta}$, let $P_{\theta,0} = P = \mathcal{N}( \mathbf{0}, \sigma_P^2 \mathbf{I})$, $P_{\theta} = Q = \mathcal{N}( \theta, \alpha^2 \mathbf{I})$, 
    then 
    \begin{align*}
        D_{KL}(Q \| P)
        &= {1 \over 2}\left\{\operatorname {tr} \left({\boldsymbol {\Sigma }}_{P}^{-1}{\boldsymbol {\Sigma }}_{Q}\right)+\left({\boldsymbol {\mu }}_{P}-{\boldsymbol {\mu }}_{Q}\right)^{\rm {T}}{\boldsymbol {\Sigma }}_{P}^{-1}({\boldsymbol {\mu }}_{P}-{\boldsymbol {\mu }}_{Q})-k+\ln {|{\boldsymbol {\Sigma }}_{P}| \over |{\boldsymbol {\Sigma }}_{Q}|}\right\}\\
        &=\frac{1}{2}\left[\frac{k \alpha^{2}+\left\|\theta\right\|_{2}^{2}}{\sigma_{P}^{2}}-k+k \ln \left(\frac{\sigma_{P}^{2}}{\alpha^{2}}\right)\right].
    \end{align*}

Let $T=\{c \exp ((1-j) / k) \mid j \in \mathbb{N}\}$ be the set of values for $\sigma_P^2$.
If for any $j \in \mathbb{N}$, the PAC-Bayesian bound 
in Lemma~\ref{lm:pac_bayes_ml}
holds for $\sigma_{P}^{2}=c \exp ((1- j) / k)$  with probability $1-\delta_{j}$ with $\delta_{j}=\frac{6 \delta}{\pi^{2} j^{2}}$, then by the union bound, all bounds w.r.t. all $\sigma_P^2 \in T$ hold simultaneously with probability at least $1-\sum_{j=1}^{\infty} \frac{6 \delta}{\pi^{2} j^{2}}=1-\delta$.

First consider $\|\theta\|^2 \leq \alpha^{2}(\exp (4 n / k)-1)$, 
then 
$k \alpha^{2}+\left\|\theta\right\|_{2}^{2} \leq k \alpha^{2}(\exp (4 n / k)+1)$.
Now set $j=\left\lfloor 1-k \ln \left(\left(\alpha^{2}+\|\theta\|_{2}^{2} / k\right) / c\right)\right\rfloor .$   By setting $c=\alpha^{2}(1+\exp (4 n / k))$, then $\ln \left(\left(\alpha^{2}+\|\theta\|_{2}^{2} / k\right) / c\right) <0$, thus we can ensure that $j \in \mathbb{N}$. Furthermore, for $\sigma_P^2 =c \exp ((1-j) / k) $, we have:
\begin{align*}
    \alpha^{2}+\|\theta\|_{2}^{2} / k 
    \leq \sigma_{P}^{2} 
    \leq \exp (1 / k)(\alpha^{2}+\|\theta\|_{2}^{2} / k )
\end{align*}
where the first inequality is derived from 
$1 - j  = \lceil k \ln ((\alpha^{2}+\|\theta\|_{2}^{2} / k) / c)\rceil \geq k \ln ((\alpha^{2}+\|\theta\|_{2}^{2} / k) / c) $, the second inequality is derived from
$1 - j  = \lceil k \ln ((\alpha^{2}+\|\theta\|_{2}^{2} / k) / c)\rceil \leq k \ln ((\alpha^{2}+\|\theta\|_{2}^{2} / k) / c) + 1$.

The KL-divergence term can be further bounded as
\begin{align*}
    D_{KL}(Q \| {P})
    &=\frac{1}{2}\left[\frac{k \alpha^{2}+\left\|\theta\right\|_{2}^{2}}{\sigma_{P}^{2}}-k+k \ln \left(\frac{\sigma_{P}^{2}}{\alpha^{2}}\right)\right] \\
    &\leq 
    \frac{1}{2}\left[\frac{k \alpha^{2}+\left\|\theta\right\|_{2}^{2}}{\alpha^{2}+\|\theta\|_{2}^{2} / k}-k+k \ln \left(\frac{\exp (1 / k)\left(\alpha^{2}+\|\theta\|_{2}^{2} / k\right)}{\alpha^{2}}\right)\right] \\
    &=
    \frac{1}{2}\left[k \ln \left(\frac{\exp (1 / k)\left(\alpha^{2}+\|\theta\|_{2}^{2} / k\right)}{\alpha^{2}}\right)\right] \\
    &=
    \frac{1}{2}\left[1+ k \ln \left(1+\frac{\|\theta\|_{2}^{2} }{k\alpha^{2}}\right)\right]. 
\end{align*}

Given the bound that corresponds to $j$ holds with probability $1-\delta_{j}$ for $\delta_{j}=\frac{6 \delta}{\pi^{2} j^{2}}$, the $\ln$ term above can be written as:
\begin{align*}
\ln \frac{1}{\delta_{j}} &=\ln \frac{1}{\delta}+\ln \frac{\pi^{2} j^{2}}{6} \\
& \leq \ln \frac{1}{\delta}+\ln \frac{\pi^{2} k^{2} \ln ^{2}\left(c /\left(\alpha^{2}+\|\theta\|_{2}^{2} / k\right)\right)}{6} \\
& \leq \ln \frac{1}{\delta}+\ln \frac{\pi^{2} k^{2} \ln ^{2}\left(c / \alpha^{2}\right)}{6} \\
& = \ln \frac{1}{\delta}+\ln \frac{\pi^{2} k^{2} \ln ^{2}(1+\exp (4 n / k))}{6} \\
& \leq \ln \frac{1}{\delta}+\ln \frac{\pi^{2} k^{2}(4 n / k)^{2}}{6}\leq \ln \frac{1}{\delta}+2 \ln (6 n).
\end{align*}

Therefore for $\epsilon \sim \mathcal{N}(\mathbf{0}, \sigma^2 \mathbf{I})$, the generalization bound is
\begin{align*}
    \mathbb{E}_{\epsilon \sim \mathcal{N}(\mathbf{0}, \sigma^2\mathbf{I})}\left[\mathcal{L}(\theta_m (\theta+{\epsilon}); \mathcal{P})\right] \leq 
    &\mathbb{E}_{\epsilon \sim \mathcal{N}(\mathbf{0}, \sigma^2\mathbf{I})}
    \left[\mathcal{L}(\theta_m (\theta+{\epsilon}); \mathcal{D})\right]+\sqrt{\frac{\frac{1}{2} k \ln \left(1+\frac{\|\theta\|_{2}^{2}}{k \sigma^{2}}\right)+\frac{1}{2}+\ln \frac{2\sqrt{n}}{\delta}+2 \ln (6 n)}{2n}}
    +\gamma_{\rm A} \\
    =
    &\mathbb{E}_{\epsilon \sim \mathcal{N}(\mathbf{0}, \sigma^2\mathbf{I})}
    \left[\mathcal{L}(\theta_m (\theta+{\epsilon}); \mathcal{D})\right]+\sqrt{\frac{\frac{1}{2} k \ln \left(1+\frac{\|\theta\|_{2}^{2}}{k \sigma^{2}}\right)+\frac{1}{2}
    + \ln 72 + \ln \frac{1}{\delta} + \frac{5}{2} \ln n }{2n}}
    +\gamma_{\rm A}.
\end{align*}

By Lemma 1 in \cite{laurent2000}, we have that for $\epsilon \sim \mathcal{N}(\mathbf{0}, \sigma^2\mathbf{I})$ and any positive $t$ :
\begin{align*}
    P\left(\|{\epsilon}\|_{2}^{2}-k \sigma^{2} \geq 2 \sigma^{2} \sqrt{k t}+2 t \sigma^{2}\right) \leq \exp (-t).
\end{align*}

Therefore, with probability $1-1 / \sqrt{n}$ we have that:
\begin{align*}
    \|\epsilon\|_{2}^{2} \leq \sigma^{2}(2 \ln (\sqrt{n})+k+2 \sqrt{k \ln (\sqrt{n})}) \leq \sigma^{2} k\left(1+\sqrt{\frac{\ln n}{k}}\right)^{2} = \alpha^{2}.
\end{align*}

At the stationary point $\hat{\theta}$ obtained by Sharp-MAML, we have
\begin{align*}
    \mathcal{L}(\theta_m (\hat{\theta}); \mathcal{P}) & \leq \mathbb{E}_{\epsilon \sim \mathcal{N}(\mathbf{0}, \alpha^2\mathbf{I})}\left[\mathcal{L}(\theta_m (\hat{\theta}+{\epsilon}); \mathcal{P})\right]
    \leq (1-1 / \sqrt{n}) \max _{\|\epsilon\|_{2} \leq \alpha} \mathcal{L}(\theta_m (\hat{\theta}+{\epsilon}); \mathcal{D})   +1 / \sqrt{n}  \\
    &+\sqrt{\frac{\frac{1}{2} k \ln \Big(1+\frac{\|\hat{\theta}\|_{2}^{2}}{k \sigma^{2}}\Big)+\frac{1}{2}
    + \ln 72 + \ln \frac{1}{\delta} + \frac{5}{2} \ln n }{2n}}
    +\gamma_{\rm A} \\
    & \leq \max _{\|\epsilon\|_{2} \leq \alpha} \mathcal{L}(\theta_m (\hat{\theta}+{\epsilon}); \mathcal{D})+
    \sqrt{\frac{ k \ln \Big(1+\frac{\|\hat{\theta}\|_{2}^{2}}{ \alpha^{2}} \Big(1+\sqrt{\frac{\ln n}{k}}\Big)^{2} \Big)+14 + 2\ln \frac{1}{\delta} + 5 \ln n }{4n}}
    +\gamma_{\rm A}
\end{align*}
where the last inequality holds due to  $1-1/\sqrt{n} \leq 1$ and Jensen's inequality.

And then consider $\|\hat{\theta}\|^2 > \alpha^{2}(\exp (4 n / k)-1)$, apparently in \eqref{eq:thm_pac_bayes_ml_main}, 
the RHS
\begin{align*}
&\max _{\|\epsilon\|_{2} \leq \alpha} \mathcal{L}(\theta_m (\hat{\theta}+{\epsilon}); \mathcal{D})+
    \sqrt{\frac{ k \ln \Big(1+\frac{\|\hat{\theta}\|_{2}^{2}}{ \alpha^{2}} \Big(1+\sqrt{\frac{\ln n}{k}}\Big)^{2} \Big)+14 + 2\ln \frac{1}{\delta} + 5 \ln n }{4n}}
    +\gamma_{\rm A}
     \\
  >  &\max _{\|\epsilon\|_{2} \leq \alpha} \mathcal{L}(\theta_m (\hat{\theta}+{\epsilon}); \mathcal{D})+
\sqrt{\frac{ 4n+14 + 2\ln \frac{1}{\delta} + 5 \ln n}{4n}}
+\gamma_{\rm A} \\
>& \max _{\|\epsilon\|_{2} \leq \alpha} \mathcal{L}(\theta_m (\hat{\theta}+{\epsilon}); \mathcal{D})+ 1 +\gamma_{\rm A}\\ \geq & \mathcal{L} (\theta_m (\hat{\theta}); \mathcal{P})    
\end{align*}
which completes the proof.
\end{proof}

\subsection{Discussion: choice of the perturbation radius $\alpha$}
The upper bound of the population loss on the RHS of \eqref{eq:thm_pac_bayes_ml_main}, is a function of $\alpha$. 
A choice of $\alpha > 0$ close to zero,  approximates the original MAML method without SAM.
We explain why SAM improves the generalization ability of MAML
by showing that for any sufficiently small $\alpha_0>0$, we can find $\alpha_1 > \alpha_0$ where the upper bound of the population loss takes smaller value than at $\alpha_0$. 

\begin{proof}
Let $c = {\|\theta\|_{2}^{2}} \big(1+\sqrt{\frac{\ln n}{k}}\big)^{2}$.
Denote 
$$g(\alpha) = \max _{\|\epsilon\|_{2} \leq \alpha} \mathcal{L}(\theta+{\epsilon}; \mathcal{D}) + \sqrt{\frac{ k \ln (1+\frac{c}{ \alpha^{2}} ) + 2\ln \frac{1}{\delta} + 5 \ln n + O(1)}{4n}} + \gamma_{\rm A}.$$
Since $0 \leq \mathcal{L}(\cdot ) \leq 1 $, it follows that
for any $0<\alpha_0 <  (\frac{c}{\exp(4n/k) - 1})^{1/2}$,
\begin{align*}
 g(\alpha_0)  \geq
\sqrt{\frac{ k \ln \big(1+\frac{c}{ \alpha_0^{2}} \big) + 2\ln \frac{1}{\delta} + 5 \ln n + O(1)}{4n}}
+\gamma_{\rm A}    .
\end{align*}

Choose 
\begin{align*}
 \alpha_1 > \Big(\frac{c}{\big(1+\frac{c}{ \alpha_0^{2} } \big) \exp(-4n/k) - 1} \Big)^{1/2}  > \Big(\frac{c}{\big(1+\frac{c}{ \alpha_0^{2} } \big) - 1} \Big)^{1/2}
 = \alpha_0
\end{align*}
then it follows that
\begin{align*}
    g(\alpha_1)
    &\leq 1 + 
    \sqrt{\frac{ k \ln \big(1+\frac{c}{ \alpha_1^{2}} \big) + 2\ln \frac{1}{\delta} + 5 \ln n + O(1)}{4n}}
    +\gamma_{\rm A}  \\ 
    &< 1 + 
    \sqrt{\frac{ k \ln \big(\big(1+\frac{c}{ \alpha_0^{2} } \big) \exp(-4n/k) \big) + 2\ln \frac{1}{\delta} + 5 \ln n + O(1)}{4n}}
    +\gamma_{\rm A} \\
    &= 1 + 
    \sqrt{\frac{ -4n + k \ln \big(1+\frac{c}{ \alpha_0^{2} } \big)  + 2\ln \frac{1}{\delta} + 5 \ln n + O(1)}{4n}}
    +\gamma_{\rm A} \\
    &\leq 
    \sqrt{\frac{ k \ln \big(1+\frac{c}{ \alpha_0^{2} } \big)  + 2\ln \frac{1}{\delta} + 5 \ln n + O(1)}{4n}}
    +\gamma_{\rm A}\\
   & \leq  g(\alpha_0)
\end{align*}
which completes the proof.
\end{proof}

\subsection{Discussion: justification on the assumption}

To obtain the generalization bound, we assume the population loss 
\begin{equation}
    \mathcal{L}(\theta_m (\hat{\theta});\mathcal{P} ) \leq \mathbb{E}_{\epsilon \sim \mathcal{N}(\mathbf{0}, \alpha^2\mathbf{I})}[\mathcal{L}(\theta_m (\hat{\theta}+{\epsilon}) ;\mathcal{P})]
\end{equation}
at the stationary point of the Sharp-MAML$_{\rm up}$ denoted by  $\hat{\theta}$.
We give some discussion next to justify this assumption.

If $\hat{\theta}$ is the local minimizer of $\mathcal{L}(\theta_m(\hat{\theta});\mathcal{D})$, then with high probability, $\|\epsilon\|^2 \leq \alpha^2$, $\mathcal{L}(\theta_m(\hat{\theta});\mathcal{D}) \leq \mathbb{E}_{\epsilon \sim \mathcal{N}(\mathbf{0}, \alpha^2\mathbf{I}) }[\mathcal{L}(\theta_m(\hat{\theta});\mathcal{D})]$.
Assume the empirically observed $\mathcal{D}$ is representative of $\mathcal{P}$ and preserves the local property  of the loss landscape $\mathcal{L}(\theta_m(\hat{\theta});\mathcal{P})$ around $\hat{\theta}$, i.e.
for $\mathcal{D}\sim \mathcal{P}$, $|\mathcal{D}| \to \infty$, $\mathcal{L}(\theta_m(\hat{\theta});\mathcal{D}) \leq \mathbb{E}_{\epsilon \sim \mathcal{N}(\mathbf{0}, \alpha^2\mathbf{I}) }[\mathcal{L}(\theta_m(\hat{\theta});\mathcal{D})]$ with high probability,
then  we have
$\mathcal{L}(\theta_m(\hat{\theta});\mathcal{P}) \leq \mathbb{E}_{\epsilon \sim \mathcal{N}(\mathbf{0}, \alpha^2\mathbf{I}) }[\mathcal{L}(\theta_m(\hat{\theta});\mathcal{P})]$.


\begin{table}[h]
\caption{Results on Omniglot (20-way 1-shot).}
\label{table3}
\begin{center}
\begin{small}
\begin{sc}
\begin{tabular}{lcccr}
\toprule
 Algorithms & Accuracy \\
\midrule
Matching Nets                         & 93.8$\%$ \\
Reptile \citep{nichol2018_reptile}       & 89.43$\%$ \\
FOMAML \citep{nichol2018_reptile}       & 89.40$\%$ \\
MAML (reproduced)                         & 91.77 $\%$ \\
Sharp-MAML$_{\rm low}$                    & 92.89  $\%$ \\
Sharp-MAML$_{\rm up}$                    & 92.96  $\%$ \\
Sharp-MAML$_{\rm both}$                     & {93.47} $\%$ \\
\bottomrule
\end{tabular}
\end{sc}
\end{small}
\end{center}
\vskip -0.1in
\end{table}

\begin{table}[h]
\caption{Results on Omniglot (20-way 5-shot).}
\label{table4}
\vskip 0.15in
\begin{center}
\begin{small}
\begin{sc}
\begin{tabular}{lc}
\toprule
 Algorithms & Accuracy \\
\midrule
Matching Nets                         & 98.50$\%$ \\
FOMAML \citep{nichol2018_reptile}       &  97.12$\%$ \\
Reptile \citep{nichol2018_reptile}       & 97.90$\%$ \\
MAML (reproduced)                             & 96.16$\%$ \\
Sharp-MAML$_{\rm low}$                    & 96.59$\%$ \\
Sharp-MAML$_{\rm up}$                    & 96.62$\%$ \\
Sharp-MAML$_{\rm both}$                     & 96.64 $\%$ \\
\bottomrule
\end{tabular}
\end{sc}
\end{small}
\end{center}
\vskip -0.1in
\end{table}

\section{Additional Experiments} \label{appendixE}
In this section, we provide additional details of the experimental set-up and present our results on the Omniglot dataset. 


{\bf Few-shot classification on Omniglot dataset.} We used the same experimental setups in \cite{Finn}. We use only one inner gradient step with 0.1 learning rate for all our experiments for training and testing. The batch size was set to 16 for the 20-way learning setting. Following \cite{Ravi}, 15 examples per class were used to evaluate the post-update meta-gradient. The values of $\alpha_{\rm low}$ and $\alpha_{\rm up}$ are chosen from the grid search on the set $\{0.05, 0.005, 0.0005, 0.00005\}$ and each experiment is run on each value for three random seeds. We choose the inner gradient steps from a set of $\{3,5,7,10\}$. The step size is chosen via the grid search from a set of $\{0.1, 0.01, 0.001\}$. For Sharp-MAML$_{\rm both}$ we use the same value of $\alpha_{\rm low}$ and $\alpha_{\rm up}$ in each experiment.  The reproduced result of MAML for the 20-way 1-shot setting is close to that of MAML++ \citep{Antoniou}. 
For the 20-way 1-shot setting, we observe a similar trend where Sharp-MAML$_{\rm both}$ achieves the best accuracy of $93.47\%$ as compared to $91.77\%$ of MAML. The performance gain of Sharp-MAML on the Omniglot dataset is not as significant as the Mini-Imagenet dataset because the former task is much simpler.

\end{document}